\documentclass{article}

\usepackage{iclr2027_conference,times}
\usepackage{amsmath,amssymb,amsthm,mathtools}
\usepackage{booktabs}
\usepackage{array}
\usepackage{longtable}
\usepackage{graphicx}
\usepackage{float}
\usepackage{algorithm}
\usepackage[noend]{algpseudocode}
\floatstyle{ruled}
\restylefloat{algorithm}
\usepackage{microtype}
\usepackage{xcolor}
\usepackage{url}
\usepackage[hidelinks]{hyperref}


\newcommand{\E}{\mathbb{E}}
\newcommand{\KL}{D_{\mathrm{KL}}}
\newcommand{\TV}{d_{\mathrm{TV}}}
\newcommand{\Normal}{\mathcal{N}}
\newcommand{\occ}{\zeta_{\sigma}}
\newcommand{\aff}{\operatorname{aff}}

\newcommand{\ind}{\mathbf{1}}

\newtheorem{theorem}{Theorem}
\newtheorem{proposition}[theorem]{Proposition}
\newtheorem{lemma}[theorem]{Lemma}
\newtheorem{corollary}[theorem]{Corollary}
\theoremstyle{definition}

\theoremstyle{remark}

\title{The Geometry of Randomized Smoothing on Feasible Sets}

\author{%
  Syed Izhan Khilji \\
  Technische Universität Wien
  \And
  Alireza Furutanpey\thanks{Corresponding author: \texttt{a.furutanpey@neverblink.eu}.} \\
  NeverBlink
  \And
  Schahram Dustdar \\
  Technische Universität Wien \\
  Universitat Pompeu Fabra
}

\iclrfinalcopy 

\begin{document}

\maketitle

\makeatletter
\def\@oddhead{}
\makeatother

\begin{abstract}
Randomized smoothing certifies the probability of a fixed output event as the
center of Gaussian noise moves. Feasibility or confidence filtering reports
label probabilities only among retained proposals, producing a ratio. Its numerator
is a fixed Gaussian event mass, while its denominator is the probability of
retention and can change with the center. Substituting this ratio into the
ordinary smoothing formula can therefore certify a ball that contains a
decision boundary.
We separate the problem into a geometric question and a certification
question. Geometry determines when conditioning preserves Gaussian
comparisons. Convex retained sets preserve the full comparison, while general
sets require geometric control of the retained law as the center moves. Without
such control, conditional probabilities imply no positive universal radius.
Joint retention-and-label probabilities always yield a valid certificate for
the same filtered predictor. A uniform covariance bound transfers divergence
certificates to the retained law and can yield larger radii even when the
Gaussian event comparison fails. Both methods admit finite-sample bounds.
For a learned image classifier with a training-selected nonconvex filter,
conditional R\'enyi bounds certify more images than joint-mass bounds without
additional model evaluations. A released confidence filter
exhibits verified label changes inside radii obtained by conditional
substitution. An application of adaptive Gaussian composition covers causal
finite-horizon executions with history-dependent center shifts under a
pathwise energy bound.
\end{abstract}

\section{Introduction}

Randomized smoothing turns predictions under Gaussian noise into a robustness
certificate \citep{cohen2019certified}. Now consider a system that discards nine
out of ten noisy inputs. If one label receives nine out of ten surviving votes,
the reported probability is $0.9$. Yet, the event that an input is both retained
and assigned that label may have a Gaussian mass of only $0.09$. The standard
smoothing proof controls the event mass, not the fraction among survivors.

Let $s_y(a)$ be the Gaussian
probability that a proposal centered at $a$ is retained and assigned label
$y$. Let $\occ(a)$ be the probability of retention. The filtered vote is
$p_y(a)=s_y(a)/\occ(a)$. Dividing all label masses by the same value preserves
their order at a given center. It does not preserve the comparison between two
centers because $\occ(a)$ changes with $a$. A large majority among retained
proposals can therefore carry little Gaussian mass and can change label inside
the radius obtained by substituting $p_y(a)$ into the ordinary formula.

This observation leads to two questions. When does conditioning preserve the
Gaussian comparison, and how can the filtered predictor be certified when it
does not? The retained distribution answers the first question. Its
one-dimensional projections determine whether event probabilities obey the
Gaussian comparison. The curvature of log retention probability controls KL
and Rényi divergence, while conditional covariance describes the same
behavior locally. Convex retained sets satisfy the strongest condition and
preserve the full event comparison.

The second question has two practical answers. Joint retained-label masses are
probabilities of fixed Gaussian events, so they always produce a sound
certificate without changing the filtered predictor or adding model
evaluations. A uniform covariance bound instead permits direct certification
from conditional probabilities through the Rényi bound of
\citet{li2019certified}. This can yield a larger radius without requiring the
full Gaussian event comparison.
Without either geometric control or joint masses, the conditional label probabilities alone cannot determine a positive universal radius.

We test both conclusions on a learned image model. A training-selected
nonconvex filter on CIFAR-10.2 admits a global covariance proof, which permits
a direct comparison between covariance and joint-mass certificates. A
separate AuditVotes case study finds independently verified label changes
inside radii computed from conditional probabilities. The first study shows
how geometric control recovers a direct certificate. The second exhibits the
failure when that control is absent. Our derivations combine truncated
exponential-family identities, strong log-concavity, and rejection-aware
smoothing \citep{nielsen2022truncated,gopi2022private,
sheikholeslami2022rejection}.

Code and retained result records are publicly available
\footnote{\url{https://github.com/izhan19717/The-Geometry-of-Randomized-Smoothing-on-Feasible-Sets}}.

\paragraph{Contributions.}
\begin{itemize}
  \item We characterize when conditioning retains Gaussian event and
  divergence comparisons. The criteria cover arbitrary fixed measurable
  retained sets and identify conditional covariance as the exact local
  quantity.
  \item We prove that conditional probabilities alone provide no positive
  universal radius. Algorithm~\ref{alg:joint-mass} gives a finite-sample
  certificate for the unchanged filtered predictor from joint masses. A
  uniform covariance bound transfers divergence certificates to the retained
  law, with an explicit spatial requirement when the bound is local.
  \item Exact counterexamples and two learned-model studies demonstrate both
  failure and recovery. A finite-horizon extension covers history-dependent
  center shifts chosen before each Gaussian proposal block under one common
  causal execution.
\end{itemize}

The single-step results compare centers under the same measurable retention
and task rules. Appendix~\ref{app:alternative-kernels} treats fixed
post-processing alternatives. Appendix~\ref{app:adaptive-trajectory} develops
the causal finite-horizon result and
Algorithm~\ref{alg:ccfs-trajectory} gives its finite-sample procedure.

\section{Filtered Gaussian smoothing}
\label{sec:problem}

Fix a noise scale $\sigma>0$ and a finite-dimensional affine space $H$ with
direction space $L=H-H$. Let $I_L$ be the identity on $L$ and let
$G_a=\Normal_H(a,\sigma^2I_L)$ be the Gaussian proposal centered at $a\in H$.
The retention rule $R:H\to\{0,1\}$ and task rule
$f:H\to\mathcal Y$ are measurable and fixed, where $\mathcal Y$ is finite and
$|\mathcal Y|\geq2$.
Define
\[
 K=\{z:R(z)=1\},\qquad E_y=K\cap\{z:f(z)=y\}.
\]
Assume $G_a(K)>0$. This holds at every finite center once it holds at one
center. Gaussian measures with equal nonsingular covariance are mutually
absolutely continuous.

The joint retained-label mass, acceptance probability, and conditional label
probability are
\[
 s_y(a)=G_a(E_y),\qquad \occ(a)=G_a(K)=\sum_y s_y(a),\qquad
 p_y(a)=\frac{s_y(a)}{\occ(a)}.
\]
The filtered classifier uses one fixed tie rule and satisfies
\begin{equation}
\label{eq:conditional-joint-order}
 g(a)=\arg\max_y p_y(a)=\arg\max_y s_y(a).
\end{equation}
Thus the denominator affects the numerical probabilities but not their order.

Write $\Phi$ for the standard normal CDF. The Gaussian event comparison states
that every measurable $E\subseteq H$ and every $a,b\in H$ satisfy
\[
 \Phi\!\left(\Phi^{-1}(G_a(E))-\frac{\|a-b\|}{\sigma}\right)
 \le G_b(E)\le
 \Phi\!\left(\Phi^{-1}(G_a(E))+\frac{\|a-b\|}{\sigma}\right),
\]
using the extended Gaussian quantiles at zero and one. A certificate of radius
$r$ at $a$ is invalid if the smoothed decision changes at some $b$ with
$\|b-a\|<r$. In particular, if $A$ and $B$ are the top and runner-up labels at
$a$, the \emph{substituted conditional radius} is
\[
 r_{\mathrm{cond}}(a)=\frac{\sigma}{2}
 \left[\Phi^{-1}\!\left(p_A(a)\right)
       -\Phi^{-1}\!\left(p_B(a)\right)\right]_+.
\]
Here $[t]_+=\max\{t,0\}$.

Whole-set conditioning gives $Q_a=G_a(\cdot\mid K)$, whose label probabilities
are $p_y(a)$. The Gaussian event comparison applies to the joint mass $s_y(a)$,
not directly to this conditional ratio. By contrast, a projection, classifier,
or randomized repair applied after the Gaussian draw acts identically at both
centers. The comparison therefore carries through by data processing.
Appendix~\ref{app:alternative-kernels} formalizes this distinction.

The perturbation changes only the center $a$. For a learned center map
$a_\theta(x)$, an observation-space certificate also needs a verified bound on
$\|a_\theta(x')-a_\theta(x)\|$. A support or filter that changes with $x$
requires a separate analysis.

\section{Occupancy controls conditional Gaussian geometry}

Define the Gaussian occupancy
\[
  \occ(a)=G_a(K),\qquad \ell(a)=\log \occ(a).
\]
For centers $a,b\in\aff(K)$, conditioned KL differs from ordinary Gaussian KL
by the exact first-order Taylor remainder of log occupancy.
\begin{equation}
\label{eq:occupancy-kl}
 \KL(Q_a\|Q_b)=\frac{\|a-b\|^2}{2\sigma^2}
 +\ell(b)-\ell(a)-\langle\nabla\ell(a),b-a\rangle.
\end{equation}
The same exponential-family calculation gives a three-point correction for
finite-order R\'enyi divergence. Write $D_\alpha$ for order-$\alpha$ R\'enyi
divergence, with $D_1=\KL$. Appendix~\ref{app:occupancy} gives the
calculation.

\begin{theorem}[Global divergence-rate criterion]
\label{thm:global-validity}
Assume $K$ is measurable with positive intrinsic measure in $H=\aff(K)$. The
following are equivalent. First, $\ell$ is concave on $H$. Second,
$\KL(Q_a\|Q_b)\le\|a-b\|^2/(2\sigma^2)$ for every ordered pair. Third,
$D_\alpha(Q_a\|Q_b)\le\alpha\|a-b\|^2/(2\sigma^2)$ for every ordered pair and
every finite $\alpha>0$, with $D_1=\KL$.
\end{theorem}

Concavity of $\ell$ characterizes the divergence bound. Preserving every
Gaussian event comparison requires the stronger one-dimensional condition
below.

\begin{corollary}[Convex conditioning preserves the Gaussian comparison]
\label{cor:convex-event}
If $K$ is convex, then every measurable $E\subseteq K$ and every $a,b\in H$
satisfy
\[
 \Phi\!\left(\Phi^{-1}(Q_a(E))-\frac{\|a-b\|}{\sigma}\right)
 \leq Q_b(E)\leq
 \Phi\!\left(\Phi^{-1}(Q_a(E))+\frac{\|a-b\|}{\sigma}\right).
\]
Consequently $r_{\mathrm{cond}}$ is sound for every fixed convex retained set.
The same conclusion holds when nominal conditional label probabilities are
replaced by familywise-valid lower and upper bounds.
\end{corollary}
This is the strongly log-concave tradeoff comparison of
\citet[Theorem~13]{gopi2022private} specialized in the intrinsic affine hull.
It is sharp because $K=H$ recovers the ordinary Gaussian comparison. The
proof and scale calculation are in Appendix~\ref{app:occupancy}.

For a general retained set, the full comparison still has an exact
one-dimensional test.
\begin{proposition}[Ordered-pair event criterion]
\label{prop:event-criterion}
Let $a\neq b$, $u=(b-a)/\|b-a\|$, $\rho=\|b-a\|/\sigma$, and
\[
 X(z)=\langle u,z-a\rangle/\sigma,
 \qquad F_c(t)=Q_c\{X\leq t\}.
\]
Every measurable $E\subseteq K$ obeys the two-sided Gaussian event comparison
from $Q_a$ to $Q_b$ if and only if, for every $t\in\mathbb R$,
\begin{equation}
\label{eq:event-probit-criterion}
 \Phi^{-1}(F_b(t))\geq\Phi^{-1}(F_a(t))-\rho.
\end{equation}
Extended quantiles cover zero and one.
\end{proposition}
The likelihood ratio $dQ_b/dQ_a$ is increasing in $X$. Neyman--Pearson
ordering therefore reduces all measurable events to the nested halfspaces in
Proposition~\ref{prop:event-criterion}. The proof is in
Appendix~\ref{app:occupancy}.

The occupancy identities also yield a local diagnostic.
\begin{proposition}[Integrated and local covariance control]
\label{prop:local-diagnostic}
Let $X_c\sim Q_c$ and write $C_c=\operatorname{Cov}_{L}(X_c)$. For
$d=b-a$,
\begin{equation}
\label{eq:integrated-covariance}
 \KL(Q_a\|Q_b)=\frac{1}{\sigma^4}\int_0^1(1-t)d^\top C_{a+td}d\,dt.
\end{equation}
If $\dim L\ge1$, then for every unit $u\in L$ and every
finite $\alpha>0$, with $D_1=\KL$,
\[
 \lim_{t\to0}
 \frac{D_\alpha(Q_a\|Q_{a+tu})}{\alpha t^2/(2\sigma^2)}
 =\frac{\operatorname{Var}\langle u,X_a\rangle}{\sigma^2}.
\]
Consequently the worst local divergence ratio is
\[
  \Lambda(a)=\lambda_{\max}\!\left(\operatorname{Cov}_{L}(X_a)\right)/\sigma^2.
\]
If $\Lambda(a)>1$, the ordinary Gaussian KL bound fails along a
corresponding direction for all sufficiently small nonzero displacements.
If $\lambda_{\max}(C_c)\leq\Lambda\sigma^2$ along the segment from $a$ to
$b$, then
$\KL(Q_a\|Q_b)\leq\Lambda\|a-b\|^2/(2\sigma^2)$.
\end{proposition}
Appendix~\ref{app:covariance-holdout} compares this local quantity with KL at
finite displacement. A local ratio above one proves failure of the ordinary
Gaussian KL bound. Certifying a decision radius additionally requires either
locating a label change or bounding covariance throughout the candidate ball.

\paragraph{The conditional substitution.}
The Gaussian comparison controls the joint event probability $G_b(E_y)$ from
the nominal mass $G_a(E_y)$. For the filtered classifier, the latter equals
$\occ(a)p_y(a)$ rather than $p_y(a)$. The lower comparison would require
$G_a(E_A)\geq \underline p_A$, while a conditional premise gives only
$G_a(E_A)\geq\occ(a)\underline p_A$. Dividing perturbed joint masses by
$\occ(b)$ preserves their order, but it cannot restore the missing nominal
factor.

\begin{proposition}[Confidence-filter counterexample]
\label{prop:confidence-filter}
Let $H=\mathbb R$, $\sigma=1$, and let the two logits at proposal $u$ be
$(0,(\log 9)u)$. Retain a proposal when its largest softmax probability is
strictly above $0.9$. Thus
$K=(-\infty,-1)\cup(1,\infty)$, with label zero on the left tail and label one
on the right tail. At center $a=-0.1$,
\[
 p_0(a)=\frac{\Phi(-0.9)}{\Phi(-0.9)+\Phi(-1.1)}
       =0.575680\ldots,
 \qquad r_{\mathrm{cond}}(a)=0.190855\ldots .
\]
The conditional decision boundary is at zero. Center $b=0.01$ therefore has
the opposite label although $|b-a|=0.11<r_{\mathrm{cond}}(a)$. The joint-mass
radius in Theorem~\ref{thm:joint-mass} equals $0.1$, the exact distance from
$a$ to the boundary.
\end{proposition}

\begin{proposition}[Compact finite-displacement witness]
\label{prop:lshape-witness}
Let
\[
 K=([0,1]\times[0,\tfrac12])\cup
   ([0,\tfrac12]\times[\tfrac12,1]),\qquad \sigma=\tfrac3{20},
\]
and take $a=(13/20,49/100)$ and $b=(66/100,49/100)$, so
$\|a-b\|=1/100$.  For
$E=K\cap\{z_1\le633/1000\}$, exact-rational outward interval
propagation certifies
\[
 Q_a(E)>0.52747658,
 \qquad Q_b(E)<0.49888514,
 \qquad Q_a(E)-\Phi(1/15)>9/10000.
\]
It also certifies
\[
 \KL(Q_a\|Q_b)-\frac{\|a-b\|^2}{2\sigma^2}>\frac1{5000},
 \qquad
 \TV(Q_a,Q_b)-\left(2\Phi\!\left(\frac{\|a-b\|}{2\sigma}\right)-1\right)
 >\frac1{500}.
\]
Consequently the binary rule $f=\ind_E$ changes its conditioned label at $b$,
although $b$ lies strictly inside $r_{\mathrm{cond}}(a)$.
\end{proposition}

\begin{figure}[H]
  \centering
  \includegraphics[width=\linewidth]{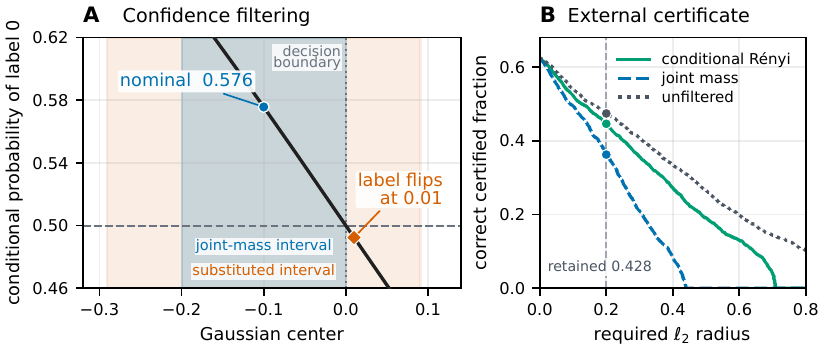}
  \caption{Normalization failure and geometry-controlled recovery. Panel A
  shows Proposition~\ref{prop:confidence-filter}. The conditional label changes
  at zero. The substituted interval crosses this boundary, while the joint-mass
  interval ends at it. Panel B uses the projected-band filter selected before
  CIFAR-10.2 access. Curves give the fraction of all test images with the
  correct selected label and a positive valid radius at least as large as the
  abscissa. Conditional R\'enyi and joint mass certify the same filtered predictor.
  Unfiltered smoothing is a separate baseline. The vertical line marks radius
  $0.2$.}
  \label{fig:exact-witness}
\end{figure}

Pr\'ekopa log-concavity also makes a convex support satisfy
Theorem~\ref{thm:global-validity} \citep{prekopa1971logarithmic}. Convexity is
sufficient but not necessary for log occupancy to be concave. The confidence
filter and the connected L-shape show that simple fixed nonconvex supports need
not satisfy either the substituted decision radius or the Gaussian divergence
rate. The outward interval algorithm and certified enclosures are in
Appendix~\ref{app:lshape-witness}. Figure~\ref{fig:exact-witness-support} shows
the retained supports.

\section{Certifying the unchanged filtered classifier}

Equation~\eqref{eq:conditional-joint-order} gives a direct correction. The
conditional predictor can be left unchanged while its certificate is computed
from the probabilities of the fixed Gaussian events $E_y$.

\begin{theorem}[Joint retained-label certificate]
\label{thm:joint-mass}
Fix the retention rule, task rule, label set, and noise scale from
Section~\ref{sec:problem}. Let $A$ be the selected label at center $a$. Suppose
numbers $L_A$ and $U_y$ satisfy
\[
  L_A\leq s_A(a),\qquad s_y(a)\leq U_y\quad\text{for every }y\neq A.
\]
Set $U_B=\max_{y\neq A}U_y$ and set $r_{\mathrm{mass}}(a)=0$ when
$L_A\leq U_B$. Otherwise, the conditional classifier is constant throughout
the open ball of radius
\begin{equation}
\label{eq:joint-mass-radius}
 r_{\mathrm{mass}}(a)=\frac{\sigma}{2}
 \left[\Phi^{-1}(L_A)-\Phi^{-1}(U_B)\right].
\end{equation}
At the population level one may take $L_A=s_A(a)$ and
$U_B=\max_{y\neq A}s_y(a)$.
\end{theorem}

The proof applies the lower Gaussian comparison to $E_A$ and the upper
comparison to each $E_y$. Their bounds remain strictly ordered inside
\eqref{eq:joint-mass-radius}. Equation~\eqref{eq:conditional-joint-order} then
transfers the result to the conditional classifier. The full proof is in
Appendix~\ref{app:joint-mass}.

\citet{sheikholeslami2022rejection} previously retained rejection as an output
and certified joint selected-label and selected-or-rejection masses.
If $L_{A\cup\bot}$ lower bounds the combined mass of the selected and rejected
outputs, then $1-L_{A\cup\bot}$ upper bounds every task competitor. Explicit
competitor bounds can be smaller in multiclass problems. Neither finite-sample
construction uniformly dominates because they divide the confidence budget
differently. Algorithm~\ref{alg:joint-mass} combines both bounds within one
confidence event.

\begin{proposition}[Sharpness and conditional-only impossibility]
\label{prop:joint-sharpness}
For every $0<s_B<s_A$ with $s_A+s_B\leq1$, there is a one-dimensional fixed
filter and binary task whose joint masses at the nominal center are $s_A$ and
$s_B$ and whose nearest decision boundary is exactly
$\frac{\sigma}{2}[\Phi^{-1}(s_A)-\Phi^{-1}(s_B)]$. Hence, no uniformly larger
binary radius can be determined from these two masses alone. In contrast, for
every $p\in(1/2,1)$ and every $\varepsilon>0$, there is a fixed filter with
conditional top probability $p$ whose decision boundary is less than
$\varepsilon$ away. Conditional probabilities and $\sigma$ alone therefore
give no positive universal radius over arbitrary fixed filters.
\end{proposition}

The sharp construction uses two oppositely oriented half-lines and rejects the
interval between them. The impossibility construction moves the two retained
tails outward while moving the nominal center toward their symmetry point.
Appendix~\ref{app:joint-mass} gives both calculations.

\paragraph{Transferring divergence certificates through conditioning.}
A covariance bound controls divergences of the conditional law even when the
Gaussian event comparison fails. It therefore permits direct use of the
R\'enyi classification bound of \citet{li2019certified}. For $p>q>0$ and
$p+q\leq1$, define
\[
 C_\alpha(p,q)=-\log\!\left(1-p-q+2M_{1-\alpha}(p,q)\right),\qquad
 M_t(p,q)=\left(\frac{p^t+q^t}{2}\right)^{1/t},
\]
with $M_0(p,q)=\sqrt{pq}$. Zero probabilities use continuous limits.
This is the least reverse categorical R\'enyi divergence permitting a tie.
\begin{proposition}[Conditional divergence certificates]
\label{prop:diameter-radii}
Fix $a$ and $R\in(0,\infty]$. Let $X_c\sim Q_c$ and write
$C_c=\operatorname{Cov}_{L}(X_c)$. Suppose
$\lambda_{\max}(C_c)\leq\Lambda\sigma^2$ for every $c$ with
$\|c-a\|<R$, where $\Lambda>0$. If $p_A(a)>p_B(a)$, the conditional decision
is fixed throughout the open ball of radius
\begin{equation}
\label{eq:conditional-renyi-radius}
 r_{\mathrm R}=\sup_{\alpha\geq1}\min\!\left\{
 \frac{R}{\alpha},\ \sigma\sqrt{\frac{2C_\alpha(p_A,p_B)}{\alpha\Lambda}}
 \right\}.
\end{equation}
For any finite set of orders containing one, its maximum is also valid.
For conditional confidence bounds, replace $p_A,p_B$ by
$L_A,\min\{U_B,1-L_A\}$, returning zero when these are not strictly ordered.
\end{proposition}

The covariance premise bounds $D_\alpha(Q_b\|Q_a)$ by
$\alpha\Lambda\|b-a\|^2/(2\sigma^2)$. Orders above one require control as far
as the tilted center $a+\alpha(b-a)$, which explains $R/\alpha$.
Covariance at $a$ alone is insufficient. The order-one case uses reverse KL.
It improves on the forward-KL radius
$r_{\mathrm{cov}}=\min\{R,\sigma\sqrt{2\mathcal J(p_A,p_B)/\Lambda}\}$,
where $\mathcal J(p,q)=p\log(2p/(p+q))+q\log(2q/(p+q))$.
Order selection reuses one confidence event and adds no model evaluations.
Appendix~\ref{app:joint-mass} proves the transfer and the improvement.

\paragraph{Certified nonconvex recovery.}
Let $\sigma=1$ and retain the two horizontal bands
$K=\mathbb R\times([-1,-0.9]\cup[0.9,1])$, labeled by band. At
$a=(0,-0.1)$ the nearest decision boundary is exactly $0.1$ away. The law
factorizes and Popoviciu gives $\operatorname{Cov}(Q_c)\preceq I$ at every
center. Thus $\Lambda=1$ is certified globally although $K$ is nonconvex and
has infinite diameter. Outward rational intervals certify
\[
 r_{\mathrm{mass}}=0.04067\ldots <
 r_{\mathrm{cov}}=0.09470\ldots <0.1<
 r_{\mathrm{cond}}=0.11888\ldots .
\]
This example shows that the ordinary Gaussian KL bound can hold even when the
full Gaussian event comparison fails. Even the forward-KL certificate is larger
than the joint-mass certificate while remaining below the exact decision
boundary. The construction extends to a unit vector $u$ learned on separate
data and then fixed in the filter
$R(x)=\ind\{\alpha\leq|u^\top x+b|\leq\beta\}$, where
$0<\alpha<\beta\leq\sigma$. The same global covariance bound holds.
Appendix~\ref{app:two-band} gives the calculation and this extension, and
Figure~\ref{fig:two-band-covariance-radius} gives a continuous sweep.

Write $\operatorname{CP}_{\rm lower}(c,n,\eta)$ and
$\operatorname{CP}_{\rm upper}(c,n,\eta)$ for one-sided Clopper--Pearson
endpoints with tail failure probability $\eta$ \citep{clopper1934confidence}.

\begin{algorithm}[H]
\caption{Filtered Gaussian certification with familywise competitor bounds}
\label{alg:joint-mass}
\begin{algorithmic}[1]
\Statex \textbf{Fixed inputs}\quad center $a$, scale $\sigma$, filter $R$, task rule $f$, labels $\mathcal Y$, tie rule
\Statex \textbf{Statistical inputs}\quad selection size $N_0$, estimation size $N$, error $\delta$
\State Draw $Z_i^{(0)}\overset{\mathrm{iid}}{\sim}G_a$ for $i=1,\ldots,N_0$
\State $C_y^{(0)}\gets\sum_i\ind\{R(Z_i^{(0)})=1,\ f(Z_i^{(0)})=y\}$ for every $y\in\mathcal Y$
\State Select $\widehat A$ from $(C_y^{(0)})_y$ using the fixed tie rule
\State Draw a fresh batch $Z_i\overset{\mathrm{iid}}{\sim}G_a$ for $i=1,\ldots,N$
\State $C_y\gets\sum_i\ind\{R(Z_i)=1,\ f(Z_i)=y\}$ for every $y\in\mathcal Y$
\State $C_\bot\gets N-\sum_{y\in\mathcal Y}C_y$
\State $L_{\widehat A}\gets\operatorname{CP}_{\rm lower}(C_{\widehat A},N,\delta/3)$
\State $L_{\widehat A\cup\bot}\gets\operatorname{CP}_{\rm lower}(C_{\widehat A}+C_\bot,N,\delta/3)$
\For{$y\in\mathcal Y\setminus\{\widehat A\}$}
  \State $U_y\gets\operatorname{CP}_{\rm upper}(C_y,N,\delta/[3(|\mathcal Y|-1)])$
\EndFor
\State $U_B\gets\min\{\max_{y\neq\widehat A}U_y,\ 1-L_{\widehat A\cup\bot}\}$
\If{$L_{\widehat A}\leq U_B$}
  \State \Return abstain
\EndIf
\State $r\gets\frac{\sigma}{2}[\Phi^{-1}(L_{\widehat A})-\Phi^{-1}(U_B)]$
\State \Return $(\widehat A,r)$
\end{algorithmic}
\end{algorithm}

With probability at least $1-\delta$, the bounds hold jointly for each input.
The algorithm computes two valid upper bounds on the strongest competitor and
uses the smaller one. One comes from explicit competitor counts and the other
from the selected-or-rejected mass \citep{sheikholeslami2022rejection}. Rejected
proposals count toward the common denominator $N$, so both bounds use the same
model evaluations. Normalizing by the retained count instead estimates
conditional probabilities and does not support the ordinary Gaussian radius.
A single-batch variant bounds every label before selection.

\paragraph{Which comparison applies.}
A fixed classifier, projection, or randomized repair is common
post-processing and uses the ordinary event certificate. Convex conditioning
uses the same certificate by Corollary~\ref{cor:convex-event}. Log-concave
occupancy gives the KL and R\'enyi rate in
Theorem~\ref{thm:global-validity}, but not by itself the event comparison. A
certified uniform covariance bound gives
$r_{\mathrm R}$. Arbitrary fixed conditioning uses joint retained-label
masses. Appendix~\ref{app:adaptive-trajectory} gives the corresponding
finite-horizon comparison for predictable Gaussian blocks under one common
causal program.

\section{Experiments}

We evaluate a geometry-controlled certificate on a learned predictor, then
measure normalization error and locate individual conditional-label changes
for a different released filter.

\paragraph{Global covariance certificate on a learned model.}
\label{sec:external-covariance}
Using CIFAR-10 training images, we selected a unit direction and offset for the
nonconvex filter
$\alpha\leq|\langle u,z\rangle+b|\leq\beta$ with
$0<\alpha<\beta=\sigma=0.25$. This support satisfies the analytic bound
$\operatorname{Cov}(Q_c)\preceq\sigma^2I$ at every center. We evaluated the
resulting filter on all $2{,}000$ CIFAR-10.2 images. All images, including
abstentions, remain in each reported denominator. The joint confidence bounds
have familywise error at most $0.001$.

At radius $0.2$, conditional R\'enyi, joint-mass, and unfiltered smoothing
correctly certify $892$, $725$, and $948$ images. The original forward-KL
analysis certifies $829$. At radius $0.5$, R\'enyi certifies $381$ images,
where forward KL and joint mass certify none and unfiltered smoothing
certifies $529$. The R\'enyi calculation is a reanalysis of the same saved
counts and confidence bounds, with no additional model evaluations.
On $877$ correctly classified images, its lower radius bound exceeds the
joint-mass upper radius bound by at least $0.01$. Figure~\ref{fig:exact-witness}B
gives the curves. Appendix~\ref{app:cifar102} reports the original experiment,
reanalysis, and a matched-model-call comparison. Unfiltered smoothing remains
stronger at these radii, including with matched model calls.

\paragraph{Released image model.}
We use the AuditVotes Gaussian image code, released CIFAR-10
\citep{krizhevsky2009learning} ResNet-110 checkpoint, and $\sigma=0.25$
\citep{lai2026auditvotes}. The filter retains a proposal when its largest
softmax probability exceeds $0.9$. We evaluate all $10{,}000$ CIFAR-10 test
images and all $2{,}021$ CIFAR-10.1 v4 images \citep{recht2018cifar101}. A
deterministic protocol specified the latter dataset, full evaluation, and
attack cohort before model inference. Each image uses independent batches of
$N_0=100$ proposals for label selection and $N=10{,}000$ for estimation at
$\delta=0.001$. Appendix~\ref{app:cifar101} gives provenance and full results.

The released certificate uses conditional label probabilities as the nominal
inputs to the Gaussian halfspace formula. Its implementation estimates those
probabilities using the retained count as the binomial denominator
\citep{lai2026auditvotes}.

For each image, every method uses the same selection and estimation batches.
The released one-sided calculation
uses $\sigma[\Phi^{-1}(\underline p_A)]_+$, where $\underline p_A$ is a
retained-count binomial lower bound. The explicit-competitor calculation
keeps $N$ as the denominator and applies Theorem~\ref{thm:joint-mass}. The
rejection-complement calculation is due to
\citet{sheikholeslami2022rejection}. Algorithm~\ref{alg:joint-mass} combines
both runner bounds in one familywise confidence event. Each curve reports
the fraction of all images with the correct selected label and radius greater
than $\varepsilon$. The joint-mass and unfiltered curves report certified
accuracy. The conditional-substitution curve is diagnostic.

\begin{figure}[t]
  \centering
  \ifdefined\arxivversion
  \includegraphics[width=0.9\linewidth]{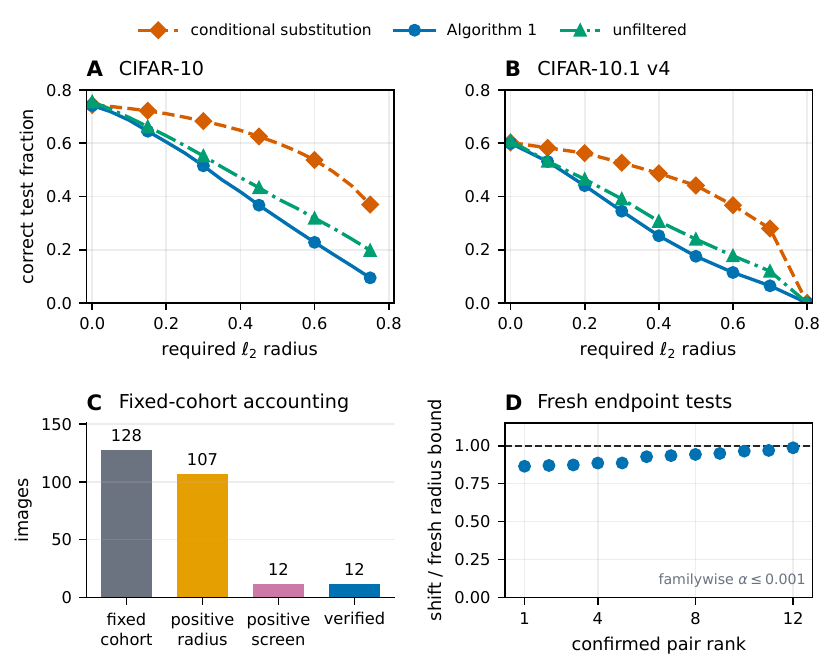}
  \else
  \includegraphics[width=\linewidth]{figures/auditvotes_recertification.pdf}
  \fi
  \caption{Released-model evaluation and fixed-cohort attack. Panels A and B
  use all CIFAR-10 and CIFAR-10.1 v4 images with $N_0=100$, $N=10{,}000$, and
  $\delta=0.001$. The diagnostic conditional-substitution curve shares proposal
  batches with the certified calculations. Panel C includes all $128$
  deterministically
  selected images. Panel D divides twelve independently confirmed shifts by
  familywise lower bounds on their substituted radii. Every ratio
  is below one, its endpoint labels differ, and familywise error is at most
  $0.001$.}
  \label{fig:auditvotes-recertification}
\end{figure}

At radius $0.5$, Algorithm~\ref{alg:joint-mass} certifies $32.02\%$ of
CIFAR-10 and $17.52\%$ of CIFAR-10.1 images. Unfiltered smoothing certifies
$39.17\%$ and $23.95\%$. The conditional substitution reports $60.07\%$ and
$44.14\%$. Mean retention is $49.34\%$ and $35.11\%$. A
nine-threshold study selects no confidence filtering on development data and
confirms the choice on untouched indices. Appendices~\ref{app:auditvotes}--
\ref{app:threshold-selection} give the complete calculations.

The endpoint study reports all $128$ images in a cohort selected without model
output. Fresh independent samples verify opposite population conditional
labels in twelve endpoint pairs. Every displacement is smaller than its
familywise lower radius bound, with displacement-to-bound ratios from $0.864$
to $0.986$. Nine cases occur among the $80$ correctly classified images with a
positive conditional-substitution radius.
Appendix~\ref{app:cifar101-attack} gives the search, joint inference,
and every confirmed case.

\paragraph{Interpreting the comparisons.}

Figure~\ref{fig:exact-witness}B compares two sound certificates for one
projected-band predictor under a global covariance proof. Figure
\ref{fig:auditvotes-recertification} studies a distinct confidence filter, for
which no covariance premise is asserted.

Algorithm~\ref{alg:joint-mass} reuses the conditional calculation's model
evaluations. Rejected proposals remain in the fixed denominator. Sampling to a
fixed number of acceptances requires stopping-rule-aware inference.

\paragraph{Sequential extension.}
We also test the finite-horizon result on SafetyPointGoal2
\citep{ji2023safetygymnasium}. Across three PPO-Lagrangian controllers and two
proposal settings per controller, the lower confidence bound on cost avoidance
remains above $0.5$ for every center-shift sequence with cumulative Euclidean
shift budget $B\leq0.02$. This certificate concerns avoidance of positive
native cost until goal or timeout. It does not establish task completion,
observation robustness, or forward invariance.

\section{Related work}

Gaussian smoothing certifies fixed output events under isotropic noise
\citep{lecuyer2019certified,li2019certified,cohen2019certified}.
Likelihood-ratio methods cover broader noise laws
\citep{dvijotham2020framework,yang2020allshapes}. Rejection-aware smoothing
treats rejection as an output
\citep{sheikholeslami2022rejection,daeubener2024uncertainty}. We retain that
joint-event foundation and study probabilities conditioned on acceptance.

Projection is fixed post-processing \citep{pfrommer2023projected}, while
known-domain methods certify bounded inputs
\citep{kou2022bounded,voracek2023box}. Neither divides by a center-dependent
acceptance probability. Input-dependent
smoothing changes the noise law \citep{sukenik2022inputdependent}, whereas
Double Sampling keeps its primary law fixed \citep{li2022double}.

Truncated exponential-family identities yield our occupancy calculation
\citep{nielsen2011renyi,nielsen2022truncated}. Work on bounded Gaussian
mechanisms studies related normalizers
\citep{chen2022boundedgaussian,fu2023truncatedrdp,hu2024boundedsupport}.
Action-constrained learning maps, masks, projects, or truncates actions
\citep{brahmanage2023flowpg,hung2025acceptance,stolz2024masking,
lee2025truncatedpolicy,stolz2026truncated,chzhen2026fab}.
Policy Smoothing certifies adaptive trajectories \citep{kumar2022policy},
and Adaptive Randomized Smoothing certifies multi-step computation
\citep{lyu2024adaptive}. Our trajectory result specializes fully adaptive
Gaussian composition \citep{smith2022adaptivegdp,koskela2023individual} to
correlated proposal blocks and a common causal feasibility program.

\section{Limitations and conclusion}

The single-step guarantees require unchanged retention and task rules.
Observation perturbations also require a verified map to Gaussian centers.
The learned-model studies use one checkpoint and noise scale. The CIFAR-10.2
experiment evaluates one training-selected filter. The AuditVotes $12/128$
result concerns a selected cohort, not dataset prevalence. Outward arithmetic
validates the CIFAR-10.2 count-to-radius calculations, not noise generation or
model execution \citep{voracek2023sound}. The trajectory theorem
certifies cost avoidance under proposal-center shifts, not task completion,
observation robustness, or forward invariance.

Filtering changes a Gaussian event probability into a ratio whose denominator
depends on the center. Convex retention controls this denominator strongly
enough to preserve the ordinary event comparison. General nonconvex retention
requires a verified geometric condition or joint retained-label events. The
resulting certificates leave the predictor unchanged.



\clearpage
\bibliography{references}
\bibliographystyle{iclr2027_conference}

\clearpage
\appendix
\raggedbottom
\section{Occupancy identities and global characterization}
\label{app:occupancy}

This appendix gives proofs and supporting arguments for the analytic statements
and the provenance of the finite nonconvex witness. All measures and
derivatives are intrinsic to the affine hull introduced in
Section~\ref{sec:problem}. No ambient density is used when the feasible set is
lower dimensional.

Write $\lambda_H$ for intrinsic Lebesgue measure on $H$ and
\[
 \varphi_{\sigma,a}(z)
 =(2\pi\sigma^2)^{-\dim L/2}
   \exp\!\left(-\frac{\|z-a\|^2}{2\sigma^2}\right)
\]
for the intrinsic Gaussian density.  Thus
$\occ(a)=\int_K\varphi_{\sigma,a}\,d\lambda_H$ and
$Q_a$ has density
$q_a(z)=\ind_K(z)\varphi_{\sigma,a}(z)/\occ(a)$.  Positive intrinsic measure
and strict positivity of the Gaussian density imply $\occ(a)>0$ for every
$a\in H$.  Let $m_a=\E_{Q_a}X_a$ and
$C_a=\operatorname{Cov}_{L}(X_a)$.

\begin{lemma}[Differentiation envelope]
\label{lem:occupancy-envelope}
Fix a compact set of centers $U\subset H$.  For each derivative order $k$,
every $k$th center derivative of $\varphi_{\sigma,a}(z)$, uniformly over
$a\in U$, is bounded in absolute value by an integrable polynomial times a
Gaussian density with a variance strictly larger than $\sigma^2$.
Consequently $\occ$ and $\ell=\log\occ$ are smooth on $H$, and their
derivatives may be taken under the integral sign.
\end{lemma}

\begin{proof}
Each center derivative is a degree-$k$ polynomial in $(z-a)/\sigma^2$ times
$\varphi_{\sigma,a}(z)$.  Fix $a_\star\in H$ and $r$ with
$U\subseteq\{a:\|a-a_\star\|\le r\}$.  If $\sigma'>\sigma$, completing the
square shows that, uniformly on this ball,
\[
 \varphi_{\sigma,a}(z)
 \le c\,\varphi_{\sigma',a_\star}(z)
\]
for a finite $c=c(\sigma,\sigma',r)$.  The polynomial factor is bounded by a
polynomial in $\|z-a_\star\|+r$, and every polynomial is integrable against
the wider Gaussian.  Restricting the integral to measurable $K$ preserves
the bound.  Differentiation under the integral follows at every order.
Continuity and positivity of $\occ$ then give smoothness of its logarithm.
\end{proof}

Differentiating once and twice gives the identities
\begin{equation}
\label{eq:occupancy-derivatives}
 \nabla\ell(a)=\frac{m_a-a}{\sigma^2},
 \qquad
 \nabla^2\ell(a)=\frac{C_a}{\sigma^4}-\frac{I_L}{\sigma^2}.
\end{equation}
Indeed,
\[
 \nabla\occ(a)
 =\frac1{\sigma^2}\int_K(z-a)\varphi_{\sigma,a}(z)\,d\lambda_H(z),
\]
and division by $\occ(a)$ yields the gradient.  A second derivative gives
\[
 \frac{\nabla^2\occ(a)}{\occ(a)}
 =\frac{\E_{Q_a}[(X_a-a)(X_a-a)^\top]}{\sigma^4}
  -\frac{I_L}{\sigma^2}.
\]
Subtracting $\nabla\ell(a)\nabla\ell(a)^\top$ yields the covariance form.

For completeness, the common-support likelihood ratio is integrable and
equals, $Q_a$-almost surely,
\[
 \log\frac{q_a(z)}{q_b(z)}
 =\frac{\|z-b\|^2-\|z-a\|^2}{2\sigma^2}
   +\ell(b)-\ell(a).
\]
Taking its $Q_a$ expectation, expanding the two squares, and using the first
identity in \eqref{eq:occupancy-derivatives} proves
\eqref{eq:occupancy-kl}.

For $\alpha\in(0,\infty)\setminus\{1\}$, the same common support gives
\begin{align*}
 \int_K q_a(z)^\alpha q_b(z)^{1-\alpha}\,d\lambda_H(z)
 &=\frac{1}{\occ(a)^\alpha\occ(b)^{1-\alpha}}
   \int_K\varphi_{\sigma,a}(z)^\alpha
             \varphi_{\sigma,b}(z)^{1-\alpha}\,d\lambda_H(z).
\end{align*}
The square-completion identity
\[
 \alpha\|z-a\|^2+(1-\alpha)\|z-b\|^2
 =\|z-(\alpha a+(1-\alpha)b)\|^2
   +\alpha(1-\alpha)\|a-b\|^2
\]
holds for every real $\alpha$.  Hence the integral is finite for every finite
positive order, including $\alpha>1$, and
\begin{equation}
\label{eq:occupancy-renyi}
 D_\alpha(Q_a\|Q_b)
 =\frac{\alpha\|a-b\|^2}{2\sigma^2}
 +\frac{
   \ell(\alpha a+(1-\alpha)b)-\alpha\ell(a)-(1-\alpha)\ell(b)}
  {\alpha-1}.
\end{equation}

\begin{proof}[Proof of Theorem~\ref{thm:global-validity}]
Suppose first that $\ell$ is concave.  Its supporting-hyperplane inequality
makes the correction in \eqref{eq:occupancy-kl} nonpositive, proving the KL
bound.  If $0<\alpha<1$, concavity makes the numerator in
\eqref{eq:occupancy-renyi} nonnegative while $\alpha-1<0$.  If $\alpha>1$,
put $c=\alpha a+(1-\alpha)b$ and observe that
\[
 a=\frac1\alpha c+\left(1-\frac1\alpha\right)b.
\]
Concavity now makes the same numerator nonpositive while $\alpha-1>0$.
Thus the R\'enyi correction is nonpositive at every finite positive order.

Conversely, the KL bound and \eqref{eq:occupancy-kl} imply
\[
 \ell(b)\le \ell(a)+\langle\nabla\ell(a),b-a\rangle
 \qquad\text{for every }a,b\in H.
\]
This is the first-order characterization of concavity for the smooth
function $\ell$.  Finally, the all-order statement contains its declared
order-one KL case, so it implies the KL statement.  This proves all three
equivalences.
\end{proof}

\begin{proof}[Proof of Corollary~\ref{cor:convex-event}]
Use orthonormal coordinates on $H$. On the convex set $K$, the density of
$Q_c$ is proportional to $\exp(-F_c)$, where
$F_c(z)=\|z-c\|^2/(2\sigma^2)$. Both $F_a$ and $F_b$ are
$\sigma^{-2}$-strongly convex. Their difference is affine with Lipschitz
constant $\|a-b\|/\sigma^2$. The strongly log-concave tradeoff comparison of
\citet[Theorem~13]{gopi2022private} therefore dominates the tradeoff curve of
two unit Gaussians separated by
$(\|a-b\|/\sigma^2)/\sqrt{\sigma^{-2}}=\|a-b\|/\sigma$.
Applying its lower bound to $E$ gives the left inequality. Applying the same
bound to $E^c$ gives the right inequality. The usual top-versus-runner
argument then gives $r_{\mathrm{cond}}$ and its simultaneous-bound version.
The zero-dimensional case is immediate.
\end{proof}

\begin{proof}[Proof of Proposition~\ref{prop:event-criterion}]
On the common support $K$,
\[
 \frac{dQ_b}{dQ_a}(z)=C\exp\!\left(\rho X(z)\right)
\]
for a positive constant $C$. The likelihood ratio is therefore increasing in
$X$. The projection $X$ has a continuous law under $Q_a$ because $K$ has
positive intrinsic measure. For a prescribed value $p\in(0,1)$, the
Neyman--Pearson lemma shows that the event with $Q_a$ mass $p$ and smallest
$Q_b$ mass is $\{X\leq F_a^{-1}(p)\}$. Thus the lower Gaussian event
comparison holds for every event exactly when
\[
 F_b(F_a^{-1}(p))
 \geq \Phi\!\left(\Phi^{-1}(p)-\rho\right)
 \qquad\text{for every }p\in(0,1).
\]
This is equivalent to \eqref{eq:event-probit-criterion}. Applying the lower
comparison to complements gives the upper comparison. Mutual absolute
continuity supplies the endpoint cases.
\end{proof}

\section{Integrated and local covariance identities}
\label{app:local}

\begin{proof}[Proof of Proposition~\ref{prop:local-diagnostic}]
First write $d=b-a$. The second-order integral remainder gives
\[
 \ell(b)-\ell(a)-\langle\nabla\ell(a),d\rangle
 =\int_0^1(1-t)d^\top\nabla^2\ell(a+td)d\,dt.
\]
Combining this identity with \eqref{eq:occupancy-kl} and
\eqref{eq:occupancy-derivatives} yields
\[
 \KL(Q_a\|Q_b)
 =\int_0^1(1-t)d^\top
 \left(\frac{I_L}{\sigma^2}+\nabla^2\ell(a+td)\right)d\,dt
 =\frac{1}{\sigma^4}\int_0^1(1-t)d^\top C_{a+td}d\,dt.
\]
The segmentwise eigenvalue bound follows immediately because
$\int_0^1(1-t)\,dt=1/2$.

Fix $a\in H$, a unit $u\in L$, and set $b=a+tu$.  Lemma
\ref{lem:occupancy-envelope} permits a second-order Taylor expansion of
$\ell$ around $a$.  For $\alpha\ne1$, the third center in
\eqref{eq:occupancy-renyi} is
$\alpha a+(1-\alpha)b=a+(1-\alpha)tu$.  The constant and linear Taylor terms
cancel, while the quadratic occupancy numerator is
\[
 \frac{t^2}{2}\bigl((1-\alpha)^2-(1-\alpha)\bigr)
 \langle u,\nabla^2\ell(a)u\rangle+o(t^2)
 =\frac{\alpha(\alpha-1)t^2}{2}
 \langle u,\nabla^2\ell(a)u\rangle+o(t^2).
\]
Substitution in \eqref{eq:occupancy-renyi} gives
\[
 D_\alpha(Q_a\|Q_{a+tu})
 =\frac{\alpha t^2}{2\sigma^2}
  \left(1+\sigma^2
   \langle u,\nabla^2\ell(a)u\rangle\right)+o(t^2).
\]
For $\alpha=1$, the same expansion follows directly from
\eqref{eq:occupancy-kl}.  Equation~\eqref{eq:occupancy-derivatives} identifies
the factor in parentheses as
\[
 1+\sigma^2\langle u,\nabla^2\ell(a)u\rangle
 =\frac{\langle u,C_a u\rangle}{\sigma^2}
 =\frac{\operatorname{Var}\langle u,X_a\rangle}{\sigma^2}.
\]
Dividing proves the claimed limit.  Maximizing this Rayleigh quotient over
unit vectors gives
$\lambda_{\max}(C_a)/\sigma^2$.  If that value is strictly larger than one,
the expansion and the definition of $o(t^2)$ make the corresponding
divergence strictly larger than its Gaussian comparator for every
sufficiently small nonzero $t$.
\end{proof}

\section{Joint-mass proofs and geometry-controlled radii}
\label{app:joint-mass}

\begin{proof}[Proof of Theorem~\ref{thm:joint-mass}]
If $L_A\leq U_B$, the stated open ball is empty. Assume $L_A>U_B$.
Let $d=\|a-b\|$. The Gaussian event comparison applied to $E_A$ gives
\[
 s_A(b)\geq
 \Phi\!\left(\Phi^{-1}(s_A(a))-d/\sigma\right)
 \geq\Phi\!\left(\Phi^{-1}(L_A)-d/\sigma\right).
\]
For every $y\neq A$, the other side of the comparison gives
\[
 s_y(b)\leq
 \Phi\!\left(\Phi^{-1}(s_y(a))+d/\sigma\right)
 \leq\Phi\!\left(\Phi^{-1}(U_B)+d/\sigma\right).
\]
The first displayed lower bound is strictly larger than the second displayed
upper bound whenever
\[
 d<\frac{\sigma}{2}
 \left[\Phi^{-1}(L_A)-\Phi^{-1}(U_B)\right].
\]
Thus $s_A(b)>s_y(b)$ for every competitor. Division by the common positive
acceptance probability $\occ(b)$ preserves these inequalities, so the
conditional classifier also selects $A$.
\end{proof}

\begin{proof}[Proof of Proposition~\ref{prop:joint-sharpness}]
First fix $0<s_B<s_A$ with $s_A+s_B\leq1$ and write
$q_A=\Phi^{-1}(s_A)$ and $q_B=\Phi^{-1}(s_B)$. At nominal center zero, let
\[
 E_A=(-\infty,\sigma q_A],\qquad
 E_B=[-\sigma q_B,\infty).
\]
The condition on the masses makes these half-lines disjoint. The interval
between them is rejected. At center $t$ their joint masses are
\[
 s_A(t)=\Phi(q_A-t/\sigma),\qquad
 s_B(t)=\Phi(q_B+t/\sigma).
\]
They are equal first at
$t=\frac{\sigma}{2}(q_A-q_B)$, exactly the population radius in
Theorem~\ref{thm:joint-mass}. This proves binary sharpness from the two masses.

For the second assertion, fix $p\in(1/2,1)$. For $T>0$, retain the symmetric
tails
\[
 K_T=(-\infty,-T]\cup[T,\infty)
\]
and assign one label to each tail. At center $-d$, the conditional probability
of the left label is
\[
 F_T(d)=
 \frac{\Phi((d-T)/\sigma)}
 {\Phi((d-T)/\sigma)+\Phi(-(T+d)/\sigma)}.
\]
The function is continuous and strictly increasing from $F_T(0)=1/2$ toward
one, so there is a unique $d_T>0$ with $F_T(d_T)=p$. Put
\[
 c_*=\frac{\sigma^2}{2}\log\frac{p}{1-p}.
\]
Gaussian tail asymptotics, applied at $d=c/T$, give uniformly for $c$ in any
fixed bounded interval
\[
 \log\frac{\Phi((c/T-T)/\sigma)}{\Phi(-(T+c/T)/\sigma)}
 =\frac{2c}{\sigma^2}+o(1).
\]
For every $\eta\in(0,c_*)$, the last display is eventually smaller than
$\log\{p/(1-p)\}$ at $c=c_*-\eta$ and larger at $c=c_*+\eta$.
Monotonicity therefore gives
\[
 c_*-\eta<Td_T<c_*+\eta
\]
for all sufficiently large $T$. Hence $Td_T\to c_*$ and $d_T\to0$. Symmetry
places the decision boundary at zero. Taking $T$ large enough makes its
distance from the nominal center smaller than any prescribed
$\varepsilon$ while keeping the same conditional probability $p$.
\end{proof}

\begin{proof}[Proof of Proposition~\ref{prop:diameter-radii}]
Let $d=b-a$ and let $\psi(c)=\|c\|^2/(2\sigma^2)+\log\occ(c)$,
using coordinates on the affine hull. Log-normalizer differentiation gives
$\nabla^2\psi(c)=\operatorname{Cov}_{L}(X_c)/\sigma^4$.
Write $g(t)=\psi(a+td)$. Whenever the complete segment from $a$ to
$a+\alpha d$ lies in the stated ball,
\[
 g''(t)\leq v,\qquad v=\Lambda\|d\|^2/\sigma^2.
\]
The exponential-family identity gives, for $\alpha>1$,
\[
 D_\alpha(Q_b\|Q_a)=
 \frac{g(\alpha)-\alpha g(1)+(\alpha-1)g(0)}{\alpha-1}
 \leq\frac{\alpha v}{2}.
\]
For the inequality, $g(t)-vt^2/2$ is concave, so its value at one is at
least the interpolation of its values at zero and $\alpha$.
The limit at one is
$\KL(Q_b\|Q_a)=g(0)-g(1)+g'(1)\leq v/2$.
Thus the order-one bound requires only the segment to $b$, while order
$\alpha>1$ requires $\alpha\|d\|<R$.

For completeness, consider nominal category probabilities
$(p,q,s)$, where $s=1-p-q$ and $p>q$. Among distributions whose second
category ties or exceeds the first, the reverse R\'enyi divergence is
minimized on the tie. This follows from convexity of
$\sum_j x_j^\alpha p_j^{1-\alpha}$ and the location of its unconstrained
minimum. Setting the tied probabilities to $t$ leaves the objective
\[
 t^\alpha(p^{1-\alpha}+q^{1-\alpha})+(1-2t)^\alpha s^{1-\alpha}.
\]
Put $m=M_{1-\alpha}(p,q)$ and $S=2m+s$. The minimizer is
$(m/S,m/S,s/S)$, and the objective is $S^{1-\alpha}$.
The minimum divergence is therefore $-\log S=C_\alpha(p,q)$,
as in \citet[Lemma~1]{li2019certified}. At order one, the same calculation
uses $m=\sqrt{pq}$ and gives the reverse-KL minimum.
Boundary probabilities follow by continuity. Data processing to the selected
label, any competitor, and their complement now rules out a tie whenever
$\alpha\Lambda\|d\|^2/(2\sigma^2)<C_\alpha(p_A,p_B)$.
This proves each radius in \eqref{eq:conditional-renyi-radius}. Their
supremum certifies the union of these concentric open balls.

For $\alpha\geq1$ and $p>q>0$, differentiating $S=1-p-q+2m$ gives
\[
 \partial_p S=-1+(m/p)^\alpha<0,\qquad
 \partial_q S=-1+(m/q)^\alpha>0.
\]
Thus $C_\alpha$ increases with $p$ and decreases with $q$. The same upper
bound on the strongest competitor covers every smaller competitor. Replacing
the probabilities by their stated confidence bounds is conservative.
Every order is a deterministic function of the same bounds, so maximizing
over orders requires no additional probability statement.
\end{proof}

\paragraph{Why reverse KL improves the forward bound.}
Let $s=p+q$ and $t=(p-q)/s$. Then
\[
 C_1(p,q)=-\log(1-s+s\sqrt{1-t^2})
 \geq s(1-\sqrt{1-t^2})\geq\mathcal J(p,q).
\]
The first inequality uses $-\log(1-x)\geq x$. For the second, the derivative
of $1-\sqrt{1-t^2}$ is $t/\sqrt{1-t^2}$, which is at least
$\operatorname{arctanh}(t)$, the derivative of
$\{(1+t)\log(1+t)+(1-t)\log(1-t)\}/2$. Both vanish at zero.
Consequently the order-one radius is never smaller than $r_{\mathrm{cov}}$
under the same covariance premise. The forward cost is at most $\log2$,
which imposes the ceiling $\sigma\sqrt{2\log2/\Lambda}$.
The reverse cost has no such ceiling as $(p,q)$ approaches $(1,0)$.

The categorical bound is an existing R\'enyi smoothing result. The
conditional application here follows from the covariance identity, with the
additional spatial requirement $R/\alpha$ for a local covariance certificate.
It does not require the Gaussian event comparison for the retained law.

\paragraph{Forward-KL and bounded-support alternatives.}
The uniform covariance premise and \eqref{eq:integrated-covariance} give
\[
 \KL(Q_a\|Q_b)\leq
 \frac{\Lambda}{2\sigma^2}\|a-b\|^2
\]
whenever the segment from $a$ to $b$ stays in the stated ball. Suppose a
competitor $y$ ties or overtakes $A$ at $b$. Data processing under the
three-cell partition formed by $E_A$, $E_y$, and their complement in $K$
gives
\[
 \KL(Q_a\|Q_b)\geq\mathcal J(p_A(a),p_y(a)).
\]
Indeed, among categorical laws in which the second cell ties or overtakes the
first, forward KL from $(p,q,1-p-q)$ is minimized at
$((p+q)/2,(p+q)/2,1-p-q)$. The attained value is $\mathcal J(p,q)$. For
$p>q>0$,
\[
 \partial_p\mathcal J=\log\frac{2p}{p+q}>0,
 \qquad
 \partial_q\mathcal J=\log\frac{2q}{p+q}<0.
\]
Thus every competitor requires KL at least
$\mathcal J(p_A(a),p_B(a))$. The covariance KL upper bound stays strictly
below this value inside $r_{\mathrm{cov}}$, which rules out a tie. The same
monotonicity permits simultaneous bounds $L_A$ and
$\widetilde U_B=\min\{U_B,1-L_A\}$ in the displayed radius. The second term is
also an upper bound because the conditional label probabilities sum to one.
If $p_B=0$, every competitor is null and Gaussian mutual absolute
continuity makes the decision constant over all centers.

If $K$ has diameter $D$, every scalar projection of a random variable supported
on $K$ has range at most $D$. Popoviciu's inequality therefore gives
\[
 \lambda_{\max}(C_c)\leq D^2/4.
\]
Taking $\Lambda=D^2/(4\sigma^2)$ gives the first of the two radii
\begin{equation}
\label{eq:diameter-radii}
 r_{\mathrm{KL}}=\frac{2\sigma^2}{D}\sqrt{2\mathcal J(p_A,p_B)},
 \qquad r_{\mathrm{odds}}=\frac{\sigma^2}{D}\log\frac{p_A}{p_B}.
\end{equation}

For the remaining radius, the likelihood ratio $dQ_b/dQ_a$ has logarithmic
oscillation at most $D\|b-a\|/\sigma^2$ over $K$. Hence
\[
 \frac{Q_b(E_A)}{Q_b(E_B)}
 \geq \exp\!\left(-\frac{D\|b-a\|}{\sigma^2}\right)
       \frac{Q_a(E_A)}{Q_a(E_B)}.
\]
The right side exceeds one inside $r_{\mathrm{odds}}$. If $p_B=0$, Gaussian
mutual absolute continuity makes the competitor null at every center, giving
the stated infinite value.

\paragraph{Finite-sample coverage.}
In Algorithm~\ref{alg:joint-mass}, the selection batch is independent of the
estimation batch. Conditional on the selected label, its lower tail and the
selected-or-rejected lower tail each fail with probability at most
$\delta/3$. The union of the $|\mathcal Y|-1$ explicit competitor upper-tail
failures has probability at most $\delta/3$. When none fail, both runner upper
bounds in Algorithm~\ref{alg:joint-mass} are valid, so their minimum is valid
and Theorem~\ref{thm:joint-mass} applies. If one batch is used for selection and
estimation, one may form all label lower bounds, all label-or-rejection lower
bounds, and all label upper bounds before selection. Assigning total error
$\delta/3$ to each family gives simultaneous coverage.

For the covariance-controlled radius, condition on a positive retained count.
The retained label counts are multinomial with probabilities $(p_y(a))_y$.
An independent selection batch or simultaneous conditional bounds therefore
provides $L_A$ and $U_B$ with the stated coverage. The covariance upper bound
needs its own simultaneous guarantee over the complete center ball. A sample
covariance at the anchor does not supply that guarantee.

This statement treats binomial endpoints and Gaussian quantiles as real
numbers. The empirical implementation uses SciPy to evaluate them. It therefore
does not by itself address adversarial floating-point execution. A deployment
requiring a machine-arithmetic guarantee must use validated probability and
quantile routines, as in sound floating-point smoothing
\citep{voracek2023sound}. The compact witness in
Appendix~\ref{app:lshape-witness} uses separate outward rational intervals.

\section{Outward-rational certificate for the compact witness}
\label{app:lshape-witness}

We prove Proposition~\ref{prop:lshape-witness} without treating floating-point
quadrature as exact.  Write $\phi$ and $\Phi$ for the standard-normal density
and CDF.  For a rectangle
$R=\prod_{k=1}^d[r_k^-,r_k^+]$, center $c$, and scale $\sigma$, define
\[
 \alpha_k=(r_k^- -c_k)/\sigma,
 \qquad \beta_k=(r_k^+ -c_k)/\sigma,
 \qquad p_k=\Phi(\beta_k)-\Phi(\alpha_k).
\]
Gaussian factorization gives $G_c(R)=\prod_k p_k$.  Because the two rectangles
in the witness have disjoint interiors, their masses add exactly.  The same
factorization gives the unnormalized first-coordinate moment
\[
 \left[c_1p_1+\sigma\{\phi(\alpha_1)-\phi(\beta_1)\}\right]
 \prod_{k\ne1}p_k.
\]
These formulas yield both occupancies and the conditional mean used in the KL
identity~\eqref{eq:occupancy-kl}.

The corresponding unnormalized second moment of the first coordinate is
\[
 \left[c_1^2p_1+2c_1\sigma\{\phi(\alpha_1)-\phi(\beta_1)\}
 +\sigma^2\{p_1+\alpha_1\phi(\alpha_1)-\beta_1\phi(\beta_1)\}\right]
 \prod_{k\ne1}p_k.
\]
Adding this quantity over the two rectangles, dividing by occupancy, and
subtracting the squared conditional mean gives the local variance ratio reported
in the certified enclosure table below.

For $b=a+de_1$ with $d>0$, the common-support likelihood ratio obeys
\[
 \log\frac{q_a(z)}{q_b(z)}
 =\frac{d}{\sigma^2}(x_*-z_1),
 \qquad
 x_* = \frac{a_1+b_1}{2}
       +\frac{\sigma^2}{d}\log\frac{\occ(b)}{\occ(a)}.
\]
Thus $q_a>q_b$ precisely to the left of $x_*$, up to a null hyperplane, and
\[
 \TV(Q_a,Q_b)
 =Q_a\{z_1\le x_*\}-Q_b\{z_1\le x_*\}.
\]

All primitive evaluations are enclosed by rational intervals.  We use the
alternating series for $e^{-y}$ and
\[
 \int_0^t e^{-x^2/2}\,dx
 =\sum_{n\ge0}\frac{(-1)^nt^{2n+1}}{2^n n!(2n+1)},
\]
after their terms decrease. Adjacent partial sums enclose the value.  For
$x>0$, we enclose
\[
 \log x=2\sum_{k\ge0}\frac{z^{2k+1}}{2k+1},
 \qquad z=\frac{x-1}{x+1},
\]
with tail at most
$2|z|^{2N+3}/((2N+3)(1-z^2))$.  Machin's identity and alternating arctangent
series enclose $\pi$.  Closed rational interval arithmetic then propagates
these bounds through rectangle masses, normalization, logarithms, and event
differences. Every divisor interval has a strictly positive lower endpoint.

At the rational data in Proposition~\ref{prop:lshape-witness}, the resulting
outward enclosures are
\begin{center}
\begin{tabular}{@{}lr@{}}
\toprule
Quantity & certified enclosure \\
\midrule
$\KL(Q_a\|Q_b)$ & $[0.002423524,\,0.002423525]$ \\
Gaussian KL comparator & $1/450=0.002222\ldots$ \\
$\operatorname{Var}_{Q_a}(Z_1)/\sigma^2$
  & $[1.091538667648430022208362,\,1.091538667648430022208363]$ \\
$\TV(Q_a,Q_b)$ & $[0.028641422,\,0.028641423]$ \\
Gaussian TV comparator & $[0.026591227,\,0.026591228]$ \\
$Q_a(E)$ & $[0.527476587,\,0.527476588]$ \\
$Q_b(E)$ & $[0.498885133,\,0.498885134]$ \\
$Q_a(E)-\Phi(1/15)$ & $[0.000900122,\,0.000900123]$ \\
$x_*$ & $[0.623175082,\,0.623175083]$ \\
\bottomrule
\end{tabular}
\end{center}
The KL excess is therefore greater than $0.0002013>1/5000$, and the TV excess
is greater than $0.0020501>1/500$. The event bounds give opposite strict binary
labels at $a$ and $b$.  Finally,
$Q_a(E)>\Phi(1/15)$ implies
\[
 \sigma\Phi^{-1}(Q_a(E))>
 \frac{3}{20}\frac1{15}=\frac1{100}=\|a-b\|,
\]
so the changed label is strictly inside the imported ambient-Gaussian radius.
This proves every assertion in Proposition~\ref{prop:lshape-witness}.

The retained verifier evaluates the same rational series and records every
remainder and propagated interval. A focused test checks the printed strict
inequalities and fails if an endpoint is interchanged. Figure
\ref{fig:exact-witness-support} shows the connected support and the certified
displacement.

\begin{figure}[H]
  \centering
  \includegraphics[width=0.88\linewidth]{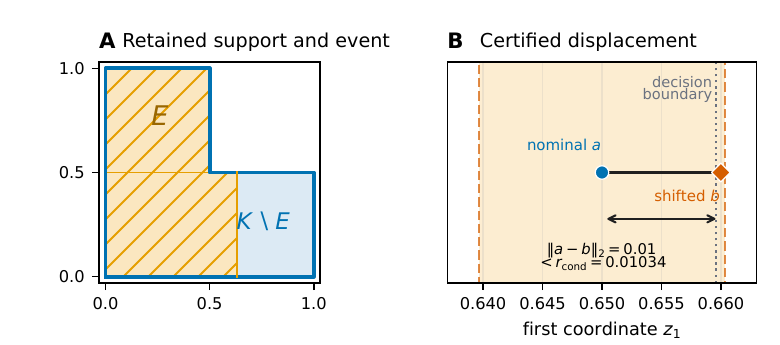}
  \caption{Compact nonconvex witness. Panel A shows the connected L-shaped
  support, the event $E$, and its complement $K\setminus E$. Panel B shows the
  tested displacement, the conditional decision boundary, and the substituted
  radius. The outward-certified bounds above establish the label change.}
  \label{fig:exact-witness-support}
\end{figure}

\section{A covariance-controlled nonconvex family}
\label{app:two-band}

Let $L=L_0\oplus L_1$ be an orthogonal decomposition and let
$K=L_0\times S$, where $S\subseteq L_1$ has positive measure and diameter
at most $2\sigma$. The conditioned Gaussian factors into an unrestricted
Gaussian on $L_0$ and a Gaussian conditioned on $S$ in $L_1$. For every center
$c$, the covariance has blocks $\sigma^2 I_{L_0}$ and
$\operatorname{Cov}(Y_c\mid Y_c\in S)$. Every unit projection of the second
block has range at most $\operatorname{diam}(S)$, so Popoviciu's inequality
gives $\operatorname{Cov}(Q_c)\preceq\sigma^2 I_L$ for every $c$. Thus the
covariance radius in Proposition~\ref{prop:diameter-radii} applies globally
with $\Lambda=1$ even for a disconnected $S$. The result does not assert the
Gaussian event comparison.

This construction gives a proof-carrying learned filter. Fix a unit vector
$u$, an offset $b$, and $0<\alpha<\beta\leq\sigma$ after training, and set
\[
 R(x)=\ind\{\alpha\leq |u^\top x+b|\leq\beta\}.
\]
For $X_c\sim\Normal(c,\sigma^2I)$, the component orthogonal to $u$ is
independent of $u^\top X_c+b$ and is unchanged by retention. The retained
scalar lies in $[-\beta,\beta]$. Popoviciu's inequality therefore gives
\[
 \operatorname{Cov}(X_c\mid R(X_c)=1)
 =\sigma^2(I-uu^\top)
  +\operatorname{Var}(u^\top X_c\mid R(X_c)=1)uu^\top
 \preceq\sigma^2I
\]
for every center $c$. The direction and offset may be learned on separate
data because the proof uses only their fixed values and $\|u\|=1$.

The same statement holds for an orthonormal matrix
$U=[u_1,\ldots,u_k]$ and the product rule
$\prod_j\ind\{\alpha_j\leq|u_j^\top x+b_j|\leq\beta_j\}$ with
$\beta_j\leq\sigma$. The Gaussian coordinates in the columns of $U$ are
independent, the retained event factorizes, and every retained coordinate has
variance at most $\beta_j^2$. The orthogonal complement remains unchanged.

Let $H=\mathbb R^2$, $\sigma=1$, and
\[
 K=\mathbb R\times\bigl([-1,-0.9]\cup[0.9,1]\bigr).
\]
The two horizontal bands receive different labels. For
$a_\delta=(0,-\delta)$, reflection gives the lower-band and upper-band joint
masses
\[
 s_-(\delta)=\Phi(1-\delta)-\Phi(0.9-\delta),
 \qquad
 s_+(\delta)=\Phi(1+\delta)-\Phi(0.9+\delta).
\]
The lower band is selected for $\delta>0$. Reflection makes the masses equal
when the second center coordinate is zero. Pointwise Gaussian likelihood
ordering makes their order strict on either side. The nearest decision
boundary to $a_\delta$ is therefore exactly $\delta$ away.

Conditioning does not affect the first coordinate, so
\[
 \operatorname{Cov}(Q_c)
 =\operatorname{diag}\!\left(1,
   \operatorname{Var}(Y\mid Y\in[-1,-0.9]\cup[0.9,1])\right).
\]
The retained second coordinate lies in $[-1,1]$. Popoviciu's inequality gives
variance at most one for every center. Hence
$\lambda_{\max}(\operatorname{Cov}(Q_c))\leq1$ globally. Proposition
\ref{prop:diameter-radii} applies with $R=\infty$ and $\Lambda=1$. Notice that
$K$ has infinite diameter, so the two diameter-based formulas in
\eqref{eq:diameter-radii} are unavailable.

This numerical witness uses the forward-KL alternative proved in
Appendix~\ref{app:joint-mass}. Writing $p_-=s_-/(s_-+s_+)$ and $p_+=1-p_-$,
the compared radii are
\[
 r_{\mathrm{cond}}=\Phi^{-1}(p_-),\qquad
 r_{\mathrm{mass}}=\frac12\left[\Phi^{-1}(s_-)-\Phi^{-1}(s_+)\right],
 \qquad
 r_{\mathrm{cov}}=\sqrt{2\mathcal J(p_-,p_+)}.
\]
At $\delta=0.1$, exact rational interval propagation gives
\begin{center}
\begin{tabular}{@{}lr@{}}
\toprule
Quantity & Certified enclosure \\
\midrule
$G_{a_\delta}(K)$ & $[0.05078446622171157326,\ 0.05078446622171157327]$ \\
$p_-(a_\delta)$ & $[0.54731840865059743984,\ 0.54731840865059743985]$ \\
$r_{\mathrm{mass}}$ & $[0.04067994841241465132,\ 0.04067994841241465133]$ \\
$r_{\mathrm{cov}}$ & $[0.09470767664273030523,\ 0.09470767664273030524]$ \\
$r_{\mathrm{cond}}$ & $[0.11888914386462333130,\ 0.11888914386462333131]$ \\
\bottomrule
\end{tabular}
\end{center}
The R\'enyi certificate is at least as large as $r_{\mathrm{cov}}$, but it
is not the quantity enclosed in this table. The strict ordering printed in
the main text follows. Figure
\ref{fig:two-band-covariance-radius} traces the four quantities. The same
covariance bound also gives the global KL and finite-order R\'enyi rates through
Theorem~\ref{thm:global-validity}. The substituted event radius nevertheless
crosses the exact boundary. Thus Gaussian divergence contraction does not
imply the Gaussian event comparison.

\begin{figure}[H]
  \centering
  \includegraphics[width=.82\linewidth]{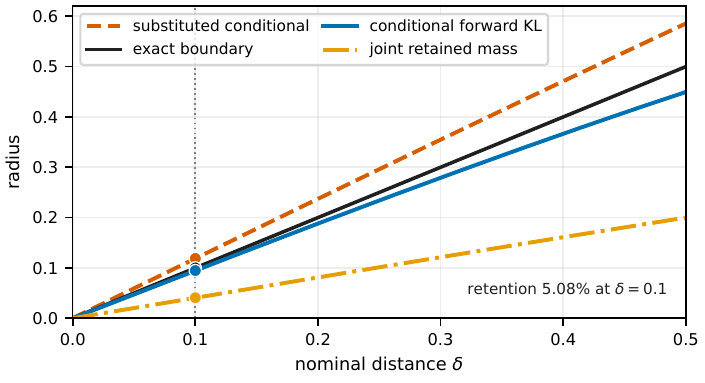}
  \caption{End-to-end radii for the nonconvex two-band family. The curves use
  the analytic Gaussian CDF formulas above. The points at $\delta=0.1$ use the
  outward-certified calculation. The forward-KL radius remains
  close to the exact boundary and is more than twice the joint-mass radius at
  the certified point. The substituted conditional radius crosses the
  boundary.}
  \label{fig:two-band-covariance-radius}
\end{figure}

The verifier uses exact fractions, alternating Gaussian-integral brackets, an
atanh logarithm bound, and outward square roots. It records all interval
endpoints and strict checks in the supplement. The displayed sweep uses
standard floating-point Gaussian CDF evaluations and is illustrative. The
certified point does not depend on those evaluations.

\section{Alternative feasible kernels}
\label{app:alternative-kernels}

A common repair channel and a fixed-weight mixture change the sampling kernel
rather than recertifying whole-set conditioning. They remain useful when the
application permits that change.

\begin{proposition}[Fixed randomized repair]
\label{prop:fixed-randomized-repair}
Let $U\sim\rho$ be independent of the center and let
$T:H\times\mathcal U\to K$ be one measurable map used unchanged at both
centers. Then $P_a=T_{\#}(G_a\otimes\rho)$ satisfies the Gaussian event and
total-variation comparisons and every finite-order Gaussian divergence bound.
This includes a fixed categorical component followed by a component-specific
projection.
\end{proposition}

Suppose $K=\bigcup_{j=1}^m K_j$, where the $K_j$ are convex subsets of one
affine space with positive intrinsic measure and every distinct pair has
zero-intrinsic-measure intersection. All component-conditioned laws
$Q_{a,j}=G_a(\cdot\mid K_j)$ use scale $\sigma$. For a center-independent
probability vector $w$, define $P_a^w=\sum_jw_jQ_{a,j}$.

\begin{theorem}[Fixed-weight comparison]
\label{thm:mode-stable}
For every finite $\alpha>0$,
\[
 D_\alpha(P_a^w\|P_b^w)\leq
 \frac{\alpha\|a-b\|^2}{2\sigma^2},\qquad D_1=\KL.
\]
The family also satisfies the Gaussian event comparison and
\[
 \TV(P_a^w,P_b^w)
 \leq2\Phi\!\left(\frac{\|a-b\|}{2\sigma}\right)-1.
\]
\end{theorem}

For a bounded common-affine partition and $\tau\in[0,1]$, let
$\pi_j(a)=G_a(K_j)/G_a(K)$, fix anchor $a_0$, and define
\[
 w_j^{(\tau,a_0)}(a)=
 \frac{\pi_j(a_0)^{1-\tau}\pi_j(a)^\tau}
 {\sum_k\pi_k(a_0)^{1-\tau}\pi_k(a)^\tau},
 \qquad
 \nu_a^{(\tau,a_0)}=\sum_jw_j^{(\tau,a_0)}(a)Q_{a,j}.
\]

\begin{proposition}[Anchored comparison]
\label{prop:anchored-comparison}
If $D_K=\sup_{x,y\in K}\|x-y\|$, then
\[
 \KL\!\left(\nu_b^{(\tau,a_0)}\middle\|\nu_c^{(\tau,a_0)}\right)
 \leq\left(\frac{1}{2\sigma^2}+\frac{\tau^2D_K^2}{8\sigma^4}\right)
 \|b-c\|^2.
\]
For $p_y^\nu(a)=\nu_a^{(\tau,a_0)}(f^{-1}(y))$, write the coefficient as
$C_\tau$. A top label $A$ and runner-up $B$ at $a$ have the valid radius
\[
 r_{\mathrm{Pin}}^\nu(a,\tau,a_0)
 =\frac{p_A^\nu(a)-p_B^\nu(a)}{\sqrt{2C_\tau}}.
\]
\end{proposition}

The endpoint $\tau=0$ has fixed weights and the Gaussian event certificate.
Positive temperatures use only the corrected bound above. The proofs are in
Appendices~\ref{app:mode-stable} and \ref{app:fixed-repair}.
Exact evaluation requires certified component occupancies, and positive
temperatures also require a certified diameter. These quantities can be
intractable for neural confidence regions. We therefore record the construction
as an alternative for applications with an explicit convex partition rather
than use it in the learned-model experiment.

\section{Fixed-weight and anchored comparisons}
\label{app:mode-stable}

We first record the two component facts used below. If $K_j$ is convex,
Theorem~\ref{thm:global-validity} and Corollary~\ref{cor:convex-event} give,
for every finite $\alpha>0$,
\begin{equation}
\label{eq:component-renyi}
 D_\alpha(Q_{a,j}\|Q_{b,j})
 \le \frac{\alpha\|a-b\|^2}{2\sigma^2}.
\end{equation}
and, for every measurable $E$,
\begin{equation}
\label{eq:component-event}
 \Phi\!\left(\Phi^{-1}(Q_{a,j}(E))-r\right)
 \le Q_{b,j}(E)\le
 \Phi\!\left(\Phi^{-1}(Q_{a,j}(E))+r\right),
 \qquad r=\frac{\|a-b\|}{\sigma}.
\end{equation}
The intrinsic scale in both statements is $r=\|a-b\|/\sigma$.

\begin{proof}[Proof of Theorem~\ref{thm:mode-stable}]
Attach the component index to a sample and define the tagged laws
\[
 \widetilde P_a(j,dz)=w_jQ_{a,j}(dz),
 \qquad
 \widetilde P_b(j,dz)=w_jQ_{b,j}(dz).
\]
For $\alpha\ne1$, let
$d_{\alpha,j}=D_\alpha(Q_{a,j}\|Q_{b,j})$.  Directly from the definition,
\[
 \exp\!\left((\alpha-1)
 D_\alpha(\widetilde P_a\|\widetilde P_b)\right)
 =\sum_jw_j\exp\!\left((\alpha-1)d_{\alpha,j}\right).
\]
Let $C_\alpha=\alpha\|a-b\|^2/(2\sigma^2)$.  Equation
\eqref{eq:component-renyi} gives $d_{\alpha,j}\le C_\alpha$.  For
$\alpha>1$, the exponential is increasing and division by $\alpha-1$ keeps
the resulting upper bound.  For $0<\alpha<1$, the exponential is decreasing,
so the sum is bounded below by
$\exp((\alpha-1)C_\alpha)$. Division by the negative $\alpha-1$ reverses the
inequality and gives the same upper bound.  At $\alpha=1$, the tagged KL chain
rule gives
\[
 \KL(\widetilde P_a\|\widetilde P_b)
 =\sum_jw_j\KL(Q_{a,j}\|Q_{b,j})
 \le \frac{\|a-b\|^2}{2\sigma^2}.
\]
Forgetting the tag is a measurable channel.  Data processing for every
positive R\'enyi order \citep{vanerven2014renyi} transfers the tagged bounds
to $P_a^w$ and $P_b^w$.

For the event statement, set
$h_r(p)=\Phi(\Phi^{-1}(p)-r)$.  For $r>0$, writing
$p=\Phi(x)$ gives
\[
 h_r'(p)=\exp(rx-r^2/2),
\]
which is increasing in $p$. Hence $h_r$ is convex.  It is the identity when
$r=0$.  Applying the lower half of \eqref{eq:component-event}, averaging, and
using Jensen's inequality gives
\[
 P_b^w(E)
 \ge\sum_jw_jh_r(Q_{a,j}(E))
 \ge h_r\!\left(\sum_jw_jQ_{a,j}(E)\right)
 =h_r(P_a^w(E)).
\]
Applying this lower bound to $E^c$ gives the upper event comparison.  Finally,
for $p=\Phi(x)$, both possible endpoint differences are maximized at the
midpoint of the two shifted Gaussian means and equal
$2\Phi(r/2)-1$.  Taking the supremum over events proves the total-variation
bound.
\end{proof}

\begin{proof}[Proof of Proposition~\ref{prop:anchored-comparison}]
Let $Z_j(x)=G_x(K_j)>0$ and let $m_{x,j}$ be the mean of $Q_{x,j}$.  From
\eqref{eq:occupancy-derivatives}, applied to component $K_j$,
\begin{equation}
\label{eq:component-log-mass-gradient}
 \nabla_x\log Z_j(x)=\frac{m_{x,j}-x}{\sigma^2}.
\end{equation}
Because the common-affine union is bounded, conditional means lie in the
closures of their convex components, and therefore
$\|m_{x,j}-m_{x,k}\|\le D_K$.

Write $v_x(j)=w_j^{(\tau,a_0)}(x)$ for the anchored categorical law.  For two
centers $b,c$, define
\[
 \delta_j=\tau\bigl(\log Z_j(c)-\log Z_j(b)\bigr).
\]
Terms common to every component cancel on normalization, so
$v_c(j)=v_b(j)e^{\delta_j}/\E_{v_b}e^{\delta_J}$.  Consequently
\begin{equation}
\label{eq:categorical-kl-mgf}
 \KL(v_b\|v_c)=\log\E_{v_b}e^{\delta_J}-\E_{v_b}\delta_J.
\end{equation}
For every $j,k$, integration of
\eqref{eq:component-log-mass-gradient} along the segment from $b$ to $c$
gives
\[
 |(\delta_j-\delta_k)|
 \le\frac{\tau D_K}{\sigma^2}\|b-c\|.
\]
Thus the range of the random variable $\delta_J$ has width at most the
right-hand side.  Hoeffding's lemma applied to
\eqref{eq:categorical-kl-mgf} yields
\begin{equation}
\label{eq:categorical-kl-bound}
 \KL(v_b\|v_c)
 \le\frac{\tau^2D_K^2}{8\sigma^4}\|b-c\|^2.
\end{equation}

Attach the component tag to $\nu_b^{(\tau,a_0)}$ and
$\nu_c^{(\tau,a_0)}$.  The tagged KL chain rule, the convex-component
order-one case of \eqref{eq:component-renyi}, and
\eqref{eq:categorical-kl-bound} give
\[
 \KL(\widetilde\nu_b\|\widetilde\nu_c)
 \le\left(\frac1{2\sigma^2}
      +\frac{\tau^2D_K^2}{8\sigma^4}\right)\|b-c\|^2.
\]
Forgetting the tag and applying KL data processing proves the asserted bound
with coefficient $C_\tau$.

Pinsker's inequality then gives, for any $a,b\in H$,
\[
 \TV(\nu_a^{(\tau,a_0)},\nu_b^{(\tau,a_0)})
 \le \sqrt{\frac{C_\tau}{2}}\,\|a-b\|.
\]
Every fixed label event changes by at most this quantity.  If $A$ is top at
$a$ and $y\ne A$, then
\[
 p_A^\nu(b)-p_y^\nu(b)
 \ge p_A^\nu(a)-p_y^\nu(a)-\sqrt{2C_\tau}\,\|a-b\|.
\]
Since $p_B^\nu(a)$ is the largest competing probability, every such
difference is strictly positive whenever
\[
 \|a-b\|<
 \frac{p_A^\nu(a)-p_B^\nu(a)}{\sqrt{2C_\tau}}.
\]
This proves the radius and independently checks its factor of two.
\end{proof}

\section{Fixed randomized repair is a common channel}
\label{app:fixed-repair}

\begin{proof}[Proof of Proposition~\ref{prop:fixed-randomized-repair}]
Let $U\sim\rho$ be auxiliary randomness whose law does not depend on the
center, and define $M_a=G_a\otimes\rho$.  The likelihood ratio between $M_a$
and $M_b$ depends only on the Gaussian coordinate. Hence
\[
 D_\alpha(M_a\|M_b)=D_\alpha(G_a\|G_b),
 \qquad
 \TV(M_a,M_b)=\TV(G_a,G_b)
\]
for every finite positive order, with the order-one convention.  Applying the
same measurable map $T:H\times\mathcal U\to K$ at both centers and using data
processing gives
\[
 D_\alpha(P_a\|P_b)
 \le D_\alpha(G_a\|G_b)
 =\frac{\alpha\|a-b\|^2}{2\sigma^2},
 \qquad 0<\alpha<\infty,
\]
with $D_1=\KL$, and
\[
 \TV(P_a,P_b)
 \le 2\Phi\!\left(\frac{\|a-b\|}{2\sigma}\right)-1.
\]

For completeness, the event statement also survives the auxiliary mixture.
For $E\subseteq K$ and fixed $u$, put
$A_u=\{z:T(z,u)\in E\}$ and
$h_r(p)=\Phi(\Phi^{-1}(p)-r)$ with $r=\|a-b\|/\sigma$.
The ordinary Gaussian event comparison gives
$G_b(A_u)\ge h_r(G_a(A_u))$.  The function $h_r$ is convex, as shown in
Appendix~\ref{app:mode-stable}. Averaging over $u$ and applying Jensen yields
\[
 P_b(E)\ge \E_\rho h_r(G_a(A_U))
          \ge h_r\!\left(\E_\rho G_a(A_U)\right)=h_r(P_a(E)).
\]
Applying the lower bound to $E^c$ gives the upper event comparison.

A deterministic repair is recovered when $\rho$ is a point mass. A
component-indexed repair takes $U$ to be a categorical component index with a
fixed center-independent law and lets $T(\cdot,U)$ be a component map fixed in
advance, such as clipping or projection.
This argument permits a state-specific channel $T_s$ when the compared centers
belong to the same fixed state fibre. It does not compare different states,
different feasible sets, center-dependent auxiliary laws, or channels that are
recomputed after the center perturbation.
\end{proof}

\section{Adaptive feasible-trajectory comparison}
\label{app:adaptive-trajectory}

At step $t$, a fixed controller maps history $h_t$ to $m_t(h_t)$. A causal
shift rule chooses $\Delta_t$ before the current proposal block. Fix
$N_t\geq1$, $\sigma_t>0$, and $0\leq\rho_t<1$. Let $\mathcal F_t$ contain the
public pre-block history and any private seed used by the shift rule. Fresh
variables $\xi_{t,0},\ldots,\xi_{t,N_t}$ are conditionally independent
standard Gaussians given $\mathcal F_t$ and are independent of
$\mathcal F_t$. The proposals are
\[
 Z_{t,k}=m_t(h_t)+\Delta_t+\sigma_t
 \bigl(\sqrt{\rho_t}\xi_{t,0}+\sqrt{1-\rho_t}\xi_{t,k}\bigr),
 \qquad
 \kappa_t=\frac{N_t}{1+(N_t-1)\rho_t}.
\]
One fixed causal program receives the complete block and returns its first
verified action, uses a finite fallback library, or abstains without a system
transition.

Policy Smoothing gives Gaussian event comparisons for adaptive trajectory
perturbations under a cumulative Euclidean budget
\citep[Theorem~1]{kumar2022policy}. Adaptive Randomized Smoothing applies
Gaussian differential privacy to adaptive multi-step computation
\citep[Theorem~2.3]{lyu2024adaptive}. Below, the fully adaptive Gaussian
composition result of \citet{smith2022adaptivegdp} and
\citet{koskela2023individual} gives a comparison under a pathwise budget.
Whitening computes the exact Mahalanobis cost of each correlated proposal
block, and the common causal feasibility program preserves the comparison.

\begin{theorem}[Adaptive feasible-trajectory comparison]
\label{thm:adaptive-trajectory}
Fix a finite horizon, a standard-Borel history space, and a Borel action subset
of $\mathbb R^d$. Let $P_0$ and $P_\Delta$ be the public-trace laws obtained
from the same measurable controller, verifier, fallback, dynamics, initial
law, and stopping rule, with the requested shifts respectively suppressed and
applied. History-dependent feasible sets are allowed. If the fresh Gaussian
innovations underlying the current block are independent of the pre-block
information and every reachable history and private shift seed obeys
\[
 \sum_t \kappa_t\|\Delta_t\|_2^2/\sigma_t^2\leq\mu^2,
\]
then every measurable public-trace event $E$ obeys
\[
 \Phi\!\left(\Phi^{-1}(P_0(E))-\mu\right)
 \leq P_\Delta(E)\leq
 \Phi\!\left(\Phi^{-1}(P_0(E))+\mu\right).
\]
\end{theorem}

\begin{proof}[Proof of Theorem~\ref{thm:adaptive-trajectory}]
Let the action dimension be $d$. For step $t$, write
\[
 C_t=(1-\rho_t)I_{N_t}+\rho_t\mathbf 1\mathbf 1^\top.
\]
The proposal block conditioned on the current history is Gaussian with
covariance $\sigma_t^2 C_t\otimes I_d$. Since $0\leq\rho_t<1$, this covariance
is nonsingular. The nominal and perturbed conditional means differ by
$\mathbf 1\otimes\Delta_t$. Whitening gives squared displacement
\begin{align*}
 r_t^2
 &=\frac{1}{\sigma_t^2}
 (\mathbf 1\otimes\Delta_t)^\top
 (C_t^{-1}\otimes I_d)
 (\mathbf 1\otimes\Delta_t) \\
 &=\frac{\mathbf 1^\top C_t^{-1}\mathbf 1}{\sigma_t^2}
 \|\Delta_t\|_2^2
 =\frac{\kappa_t\|\Delta_t\|_2^2}{\sigma_t^2}.
\end{align*}
The last equality follows because $\mathbf 1$ is an eigenvector of $C_t$ with
eigenvalue $1+(N_t-1)\rho_t$.

The shift rule chooses $\Delta_t$ from the previously visible history and any
private randomness before the current Gaussian block is generated. Couple its
private seed under both laws and include it in the latent filtration. Write
$u_t$ for the whitened mean displacement and $X_t$ for the corresponding
standard-normal block under the nominal law. Then $u_t$ is predictable and
$\|u_t\|^2=r_t^2$. After the last block, append one unobserved scalar Gaussian
with predictable displacement
\[
 u_* = \left(\mu^2-\sum_t r_t^2\right)^{1/2}.
\]
This quantity is real by the pathwise energy condition. The enlarged
experiment has total squared displacement exactly $\mu^2$.

Under its nominal law, define
\[
 M=\sum_t\langle u_t,X_t\rangle+u_*X_*.
\]
Iterated conditional expectation of the Gaussian moment generating function
gives, for every $\lambda\in\mathbb R$,
\[
 \mathbb E_0\exp\!\left(\lambda M-\frac{\lambda^2\mu^2}{2}\right)=1.
\]
Thus $M\sim\Normal(0,\mu^2)$. The sequential Gaussian density ratio of the
enlarged shifted experiment with respect to the enlarged nominal experiment is
\[
 \exp\!\left(M-\frac{\mu^2}{2}\right).
\]
For any event of nominal probability $p$, Neyman--Pearson ordering places the
smallest shifted probability on the lower tail of $M$ and the largest on its
upper tail. Direct evaluation of these two Gaussian tails gives
\[
 \Phi\!\left(\Phi^{-1}(p)-\mu\right)
 \quad\text{and}\quad
 \Phi\!\left(\Phi^{-1}(p)+\mu\right).
\]
An event that ignores the appended scalar is an event of the enlarged
experiment, so the same bounds hold before the scalar is appended. Forgetting
the private seed is also common post-processing.

At a fixed generated history, proposal ordering, clipping, feasibility tests,
fallback construction, dynamics, process randomness, and stopping use the
same conditional kernel under both laws. Their causal composition is common
post-processing of the Gaussian blocks. This remains true when the feasible
action set changes with the generated state. It fails if the perturbation
changes the verifier, fallback, dynamics, or feasible set through another
input. After stopping, append zero-shift dummy blocks to the fixed horizon.
Standard-Borel history and action spaces ensure the required regular
conditional kernels exist.

For a measurable public-trace event $E$, the Gaussian tradeoff gives
\[
 \Phi\!\left(\Phi^{-1}(P_0(E))-\mu\right)
 \leq P_\Delta(E)\leq
 \Phi\!\left(\Phi^{-1}(P_0(E))+\mu\right).
\]
Applying the same comparison to all measurable events yields
$\TV(P_0,P_\Delta)\leq2\Phi(\mu/2)-1$. The enlarged likelihood-ratio experiment
also has R\'enyi divergence $\alpha\mu^2/2$ in both directions at every order
$\alpha>1$. Therefore data processing gives
\[
 D_\alpha(P_\Delta\|P_0)\leq\frac{\alpha\mu^2}{2},\qquad
 D_\alpha(P_0\|P_\Delta)\leq\frac{\alpha\mu^2}{2}.
\]
Letting $\alpha\downarrow1$ gives both directed KL bounds $\mu^2/2$.

Suppose $\sigma_t=\sigma$ and $\kappa_t=\kappa$. Every perturbation with
$\sum_t\|\Delta_t\|_2^2\leq B^2$ has $\mu\leq\sqrt{\kappa}B/\sigma$. Replacing
$P_0(E)$ by a simultaneous lower confidence bound $p_L$ obtained from
independent nominal trajectories and solving
\[
 \Phi\!\left(\Phi^{-1}(p_L)-\frac{\sqrt{\kappa}B}{\sigma}\right)\geq q
\]
for a target probability $q\in(0,1)$ gives
\[
 B_{\rm cert}=\frac{\sigma}{\sqrt{\kappa}}
 \left[\Phi^{-1}(p_L)-\Phi^{-1}(q)\right]_+.
\]
When $p_L\geq q$, this radius inherits the coverage of $p_L$. When $p_L<q$,
the target probability is not certified even at $B=0$.

The public trace includes every externally visible selection index,
fallback result, and abstention symbol. If abstention occurs, the stopping rule
does not execute a physical action. An event requiring no abstention therefore
charges finite-library exhaustion directly. The result bounds that event's
probability. It does not prove that the feasible-action fibre is nonempty or
that the physical dynamics are forward invariant.
\end{proof}

\paragraph{Executable certification procedure.}
Fix before sampling the horizon, event $E$, schedules
$(N_t,\sigma_t,\rho_t)$, controller, candidate map, verifier, ordered finite
fallback rule $\mathcal L_t(h_t)$, transition kernel, and stopping rule. For
$N_t>1$, a desired multiplier $\lambda_t\in(1,N_t]$ can be set exactly by
\[
 \rho_t=\frac{N_t/\lambda_t-1}{N_t-1},
 \qquad \kappa_t=\lambda_t.
\]
For $N_t=1$, use $\rho_t=0$ and $\kappa_t=\lambda_t=1$. An attacked comparison may
choose $\Delta_t$ from the pre-block history, but must commit it before the
current proposal block and cannot alter another program input.
Use separate proposal, attack, and transition random streams. Do not expose
the proposal or transition stream to the shift rule. Evaluate that rule on an
isolated copy of the pre-block history.
For each of $M$ independent nominal trajectories, set $\Delta_t=0$. At history
$h_t$, compute $m_t(h_t)$ and draw the complete correlated proposal block.
Apply the common candidate map and verifier in index order. If every proposal
fails, inspect $\mathcal L_t(h_t)$ in its declared order. Record
$\bot$ and stop without a system transition when the library is exhausted.
Otherwise execute the selected verified action through the common transition
kernel. Record the selected proposal or fallback index, every system outcome,
and any abstention.

Let $X_i$ indicate whether trace $i$ satisfies $E$, and let $p_L$ be a valid
one-sided lower confidence bound for $P_0(E)$. Iid initial-law draws permit an
exact Clopper--Pearson bound based on $\sum_iX_i$. A fixed finite cohort with
one independent draw at each support point permits a Hoeffding bound for the
uniform cohort average. For target $q\in(0,1)$, define
\[
 \mu_{\rm cert}=\Phi^{-1}(p_L)-\Phi^{-1}(q)
\]
only when $p_L\geq q$. No radius is certified when $p_L<q$. For constant
$\sigma_t=\sigma$ and $\kappa_t=\kappa$, return
$B_{\rm cert}=\sigma\mu_{\rm cert}/\sqrt{\kappa}$. With varying schedules,
retain the certified ellipsoid
$\sum_t\kappa_t\|\Delta_t\|_2^2/\sigma_t^2\leq\mu_{\rm cert}^2$.
The event must be fixed independently of these nominal samples or covered by
simultaneous inference. The energy condition must hold for every reachable
history and attack seed. One observed trace cannot establish this condition.
Algorithm~\ref{alg:ccfs-trajectory} gives the complete certification
procedure. The supplement includes tested implementations of the proposal
block, totalized selector, finite trajectory rollout, energy calculation, and
count-to-radius calculation.

\begin{algorithm}[H]
\caption{Capped feasible rollout and trajectory-event certification}
\label{alg:ccfs-trajectory}
\begin{algorithmic}[1]
\Statex \textbf{Fixed program}\quad horizon $T$, event $E$, schedules $(N_t,\sigma_t,\rho_t)$, controller $m_t$
\Statex \hspace{\algorithmicindent}verifier $V_t$, ordered fallback library $\mathcal L_t$, dynamics, stopping rule
\Statex \textbf{Statistical inputs}\quad trials $M$, valid lower-bound rule $\mathcal B$, error $\delta$, target probability $q$
\For{independent nominal trials $i=1,\ldots,M$}
  \For{$t=0,\ldots,T-1$}
    \State Compute $m_t(h_t^{(i)})$ and draw the complete correlated proposal block
    \State Select the first proposal accepted by $V_t(h_t^{(i)},\cdot)$
    \State If none is accepted, inspect $\mathcal L_t(h_t^{(i)})$ in its fixed order
    \State If no fallback is accepted, record $\bot$ and stop without a transition
    \State Otherwise execute the selected action, record the outcome, and update the history
  \EndFor
\EndFor
\State $X_i\gets\ind\{\text{public trace }i\in E\}$ for every $i$
\State $p_L\gets\mathcal B(X_1,\ldots,X_M,\delta)$
\If{$p_L<q$}
  \State \Return no positive certificate
\EndIf
\State $\mu_{\rm cert}\gets\Phi^{-1}(p_L)-\Phi^{-1}(q)$
\State \Return predictable shifts satisfying
\Statex \hspace{\algorithmicindent}$\sum_t\kappa_t\|\Delta_t\|_2^2/\sigma_t^2\leq\mu_{\rm cert}^2$ on every reachable history
\end{algorithmic}
\end{algorithm}

\subsection{Learned-controller evaluation}
\label{app:goal2-trajectory}

We trained three PPO-Lagrangian controllers for $10$ million environment steps
each and retained their terminal actor and observation-normalizer files. The
controller seeds are $1120275774$, $1346831933$, and $658990255$. Evaluation
uses SafetyPointGoal2-v0 in Safety-Gymnasium 1.0.0
\citep{ji2023safetygymnasium}. One observation-only verifier restricts the two
action coordinates. A deterministic clipped-controller action is used when no
proposal passes. The verifier is a fixed heuristic and is not a native-cost
oracle.

An earlier study used goal-before-cost as its event and did not certify
probability $0.5$. Those outcomes do not enter this calculation. A $16$-seed
pilot then informed the event below. The final $256$-seed cohort was specified
before evaluation. The audit reconstructs these seeds and verifies that they
are disjoint from the earlier confirmation cohort and all $36$ development
labels.

The certified event is no positive native cost until the first goal or $500$
executed steps. A positive cost, an early environment stop, or a nonfinite
transition is failure. Reaching the goal stops the event, so no claim concerns
later states. A timeout without cost is success. The target initial law is the
uniform mixture over the $256$ fixed reset seeds. We draw one independent
Gaussian trajectory for each seed and arm. If $S$ trajectories succeed, the
simultaneous lower bound is
\[
 p_L=\left[\frac{S}{256}-
 \sqrt{\frac{\log(6/0.001)}{2\cdot256}}\right]_+.
\]
Hoeffding's inequality applies to the independent, nonidentical Bernoulli
outcomes and targets the uniform-cohort average. It does not extend the result
to other reset distributions.

The correlated-four arm has $N=4$, $\sigma=0.2$, $\rho=59/63$, and
$\kappa=1.05$. The matched-one arm has $N=1$,
$\sigma=0.2/\sqrt{1.05}$, $\rho=0$, and $\kappa=1$. Both have Gaussian energy
charge $\kappa/\sigma^2=26.25$. The protocol records $0.2$ and $1.05$
as shared base values. The post-run integrity record distinguishes the literal
schedules above. The executed schedules, energy calculations, and radii use
the same charge.

At each step, the tested shift points opposite the clipped controller center
and has norm $0.02/\sqrt{500}$. The empirical shifted counts in
Table~\ref{tab:goal2-trajectory} are descriptive. The certificate uses the
nominal lower bound and Theorem~\ref{thm:adaptive-trajectory}, hence it covers
every predictable proposal-center sequence with cumulative Euclidean budget
at most $B=0.02$ under the same causal program.

\begin{table}[H]
\centering
\caption{SafetyPointGoal2 trajectory evaluation. Safe is the certified
event count. Goal and timeout partition that count. Cost is the nominal
positive-cost count. $L_{.02}$ is the familywise lower event probability at
$B=0.02$. Shift safe is the descriptive count under the tested shift. C1, C2,
and C3 follow the controller-seed order in the text.}
\label{tab:goal2-trajectory}
\begin{tabular}{@{}llrrrrrrrr@{}}
\toprule
Ctrl & Proposals & Safe & Goal & Timeout & Cost & $p_L$ & $B_{\rm cert}$ & $L_{.02}$ & Shift safe \\
\midrule
C1 & one  & 211 & 78  & 133 & 45 & $.694$ & $.0989$ & $.657$ & 213 \\
C1 & four & 210 & 85  & 125 & 46 & $.690$ & $.0968$ & $.653$ & 214 \\
C2 & one  & 187 & 60  & 127 & 69 & $.600$ & $.0495$ & $.560$ & 180 \\
C2 & four & 189 & 59  & 130 & 67 & $.608$ & $.0535$ & $.568$ & 187 \\
C3 & one  & 205 & 107 & 98  & 51 & $.670$ & $.0861$ & $.633$ & 204 \\
C3 & four & 207 & 106 & 101 & 49 & $.678$ & $.0903$ & $.641$ & 202 \\
\bottomrule
\end{tabular}
\end{table}

All six arms certify event probability above $0.5$ at $B=0.02$ with familywise
error at most $0.001$. Only $59$--$107$ nominal trajectories per arm reach the
goal before cost. Most certified successes are therefore timeouts. This result
does not certify task completion, observation perturbations, changes to the
verifier or dynamics, or forward invariance.

The raw result stores aggregate trajectory records rather than full per-step
traces. A post-run audit checks the record schema, all $3{,}072$ arm keys,
every reported summary, seed separation, finite energy, method schedules, and
source closure. It cannot replay each selection from the stored aggregate
records. The theorem's public trace is the arm-blind controller, selection,
fallback, dynamics, and stopping output. Attack budgets and energy diagnostics
are separate experimental metadata. The supplement includes the protocol,
results, audit, checksums, and environment versions.

\section{Relation of the formal results to prior work}
\label{app:novelty-matrix}

Table~\ref{tab:novelty-boundary} identifies the earlier result used by each
formal step and its application to conditional laws, their decision rules,
and center-dependent occupancy.

\begingroup
\begin{longtable}{@{}
>{\raggedright\arraybackslash}p{\dimexpr.21\linewidth-8pt\relax}
>{\raggedright\arraybackslash}p{\dimexpr.35\linewidth-8pt\relax}
>{\raggedright\arraybackslash}p{\dimexpr.44\linewidth-8pt\relax}@{}}
\caption{Result-level relation to prior work. The middle column gives the
earlier statement used here. The last column gives the result established in
this paper.}\label{tab:novelty-boundary}\\
\toprule
Current result & Used here & Established here \\
\midrule
\endfirsthead
\toprule
Current result & Used here & Established here \\
\midrule
\endhead
\bottomrule
\endfoot
Occupancy divergences
& Fixed-support R\'enyi log-partition identities and truncated-family KL and
skew-Jensen identities \citep{nielsen2011renyi,nielsen2022truncated}
& Gaussian occupancy specialization followed by the equivalence between
concave log occupancy and the stated ordered-pair KL and positive-order
R\'enyi rates \\

Covariance control
& Log-normalizer differentiation supplies standard moment and Fisher
identities
& Their specialization gives the finite-path KL integral, local R\'enyi rate,
and top-covariance diagnostic used here \\

Corollary~\ref{cor:convex-event}
& Strongly log-concave tradeoff comparison of
\citet[Theorem~13]{gopi2022private}
& Intrinsic affine-hull specialization with displacement scale
$\|a-b\|/\sigma$ \\

Proposition~\ref{prop:event-criterion}
& Neyman--Pearson ordering for the monotone likelihood ratio between two
conditioned Gaussian centers
& Necessary and sufficient projected-mass inequalities for every measurable
retained event under an arbitrary fixed set \\

Theorem~\ref{thm:joint-mass} and Algorithm~\ref{alg:joint-mass}
& Joint selected-label and selected-or-rejection masses
\citep[Theorem~3.1]{sheikholeslami2022rejection}
& Explicit competitor bounds and their simultaneous combination with the
rejection-complement bound for arbitrary fixed filters \\

Limits of conditional votes
& The cited joint-event certificate does not provide a radius from
conditional votes alone
& Sharpness of the joint-mass radius and the support-free conditional-only
impossibility result \\

Proposition~\ref{prop:diameter-radii}
& Categorical R\'enyi reversal bound \citep[Lemma~1]{li2019certified}, data
processing, Popoviciu's inequality, and likelihood-ratio oscillation
& Transfer to conditional laws under a uniform covariance bound, with the
$R/\alpha$ restriction when that bound holds only in a ball of radius $R$ \\

Theorem~\ref{thm:adaptive-trajectory}
& Adaptive trajectory and multi-step smoothing
\citep{kumar2022policy,lyu2024adaptive}, and fully adaptive Gaussian composition
\citep{smith2022adaptivegdp,koskela2023individual}
& Exact Mahalanobis cost for correlated proposal blocks followed by one common
causal selector, verifier, fallback, dynamics, and stopping program \\
\end{longtable}
\endgroup

\section{Experimental protocol and additional results}
\label{app:experiments}

\subsection{Released-model recertification}
\label{app:auditvotes}

We audit Theorem 2 in Section 5.2 and its proof in Appendix A.2 of the retained
AuditVotes v3 manuscript. The public
source\footnote{\url{https://github.com/Yuni-Lai/AuditVotes-Certified-Robustness}}
is identified by its exact revision in the supplement. The supplement also
authenticates the manuscript, released CIFAR-10 checkpoint, and inspected
\texttt{core.py}.
The upstream repository has no license file, so neither its source nor its
checkpoint is redistributed.

The retained manuscript defines the image predictor by conditional label
probabilities in Eq.~(15). Its Theorem~2 uses bounds on those probabilities
in the Gaussian radius, while Appendix~A.2 applies the Gaussian halfspace
comparison to joint retention-and-label indicators. In the released
\texttt{core.py}, lines 97--103 of \texttt{Smooth.\_sample\_noise} filter
predictions by confidence before counting them. Lines 47--55 of
\texttt{Smooth.certify} then replace the requested proposal count by the sum
of retained class counts, compute a binomial lower bound with that
denominator, and insert the bound into \texttt{sigma*norm.ppf}. The bound
therefore estimates a conditional
probability, whereas the joint indicator used in the halfspace comparison
has nominal mass $\occ(a)p_y(a)$. This trace concerns the Gaussian image
extension, not the separate graph certificate.
The released multiclass routine uses the one-sided radius
$\sigma[\Phi^{-1}(\underline p_A)]_+$, whereas
$r_{\mathrm{cond}}$ in Section~\ref{sec:problem} is the top-versus-runner
population expression. For $p_A(a)>1/2$, the one-sided population radius is
no larger than $r_{\mathrm{cond}}(a)$ because $p_B(a)\leq1-p_A(a)$.
All learned-model plots and endpoint tests use the released one-sided
calculation.

We use PyTorch 2.3.1, torchvision 0.18.1, and one NVIDIA RTX 4070 Ti. The full
test run uses $100$ proposals per image in the selection phase and $10{,}000$
in the estimation phase, using batches of $1024$. Random streams are derived
from a declared base seed and image index. The run takes $7562$ seconds.
The supplement includes all $10{,}000$ per-image records and the analysis
script. We do not call this an exact reproduction of the published table.
The released code does not seed the PyTorch proposal stream and the original
stream is unavailable. Instead, this is a seeded full-test evaluation
of the released checkpoint and filtering rule under the published sample
sizes.

The executed input path follows the pinned image code. CIFAR-10 test images are
converted to tensors in $[0,1]^d$. Isotropic noise is added in these pixel
coordinates before the model's fixed channel normalization. The noisy tensor
is not clipped. The ResNet logits determine the class by their first maximum,
and a proposal is retained only when its largest softmax probability is
strictly greater than $0.9$. The audited source functions are
\texttt{Smooth.certify} and \texttt{Smooth.\_sample\_noise} in the pinned
\texttt{core.py}. Our evaluator reproduces that proposal law and tie rule while
using explicit seeded streams and retaining both filtered and unfiltered
counts.

Table~\ref{tab:auditvotes-additional} separates three sound finite-sample
calculations. The explicit calculation reserves half of the error probability
for the selected lower bound and half for simultaneous competitor upper bounds.
The rejection-complement calculation reserves half for the selected lower bound
and half for the selected-or-rejected lower bound. These two calculations are
preselected separately. Algorithm~\ref{alg:joint-mass} combines both runner
bounds and assigns one third of the error probability to each confidence-event
family.

\begin{table}[H]
\centering
\caption{Additional full-test results. Positive is the fraction receiving a
strictly positive radius. Median is over all test images. The remaining columns
are correct fractions with radius strictly above the displayed value.}
\label{tab:auditvotes-additional}
\begin{tabular}{@{}lrrrrr@{}}
\toprule
Calculation & Positive & Median & $0.3$ & $0.5$ & $0.7$ \\
\midrule
Conditional substitution & $.8941$ & $.6668$ & $.6835$ & $.6007$ & $.4406$ \\
Joint complement & $.8868$ & $.3342$ & $.5163$ & $.3221$ & $.1416$ \\
Joint explicit & $.8872$ & $.3293$ & $.5114$ & $.3137$ & $.1323$ \\
Joint simultaneous & $.8875$ & $.3330$ & $.5153$ & $.3202$ & $.1394$ \\
Unfiltered Gaussian & $.9122$ & $.3924$ & $.5525$ & $.3917$ & $.2403$ \\
\bottomrule
\end{tabular}
\end{table}

The selected-label accuracy is $0.7603$ after filtering and $0.7867$ without
filtering. The mean retained fraction is $0.493386$. All reported radius
thresholds use a strict comparison.

The confidence statements above are mathematical real-arithmetic statements.
The evaluation uses conventional SciPy beta and Gaussian quantiles. We record
the numerical scope rather than claim adversarial floating-point soundness
\citep{voracek2023sound}. Every count file and post-processing script is retained
so a validated quantile implementation can be substituted without reevaluating
the network.

\subsection{CIFAR-10.1 evaluation}
\label{app:cifar101}

The CIFAR-10.1 v4 evaluation protocol was specified before model inference.
This version contains $2{,}021$ images and was the first version on which its
creators evaluated classifiers \citep{recht2018cifar101}. The supplement records
the exact dataset revision and file checksums. The image and label files are not
redistributed.

The protocol specifies the full evaluation, a hash-ranked attack cohort, the
random streams, and all confirmation tests. The model, filter, noise scale,
sample sizes, batching, and no-clipping proposal law match the CIFAR-10
evaluation. The run uses an independent base seed and took $1636$ seconds on
the same GPU with PyTorch 2.2.2 and torchvision 0.17.2. The supplement includes
the protocol, implementation, per-image records, and summary.

\begin{table}[H]
\centering
\caption{Full CIFAR-10.1 v4 results. Positive is the fraction receiving a
strictly positive radius. Median is over all $2{,}021$ images. The remaining
columns are correct fractions with radius strictly above the displayed value.}
\label{tab:cifar101-additional}
\begin{tabular}{@{}lrrrrr@{}}
\toprule
Calculation & Positive & Median & $0.3$ & $0.5$ & $0.7$ \\
\midrule
Conditional substitution & $.8461$ & $.5193$ & $.5270$ & $.4414$ & $.2796$ \\
Joint simultaneous & $.8362$ & $.2122$ & $.3449$ & $.1752$ & $.0643$ \\
Joint complement & $.8362$ & $.2125$ & $.3459$ & $.1757$ & $.0663$ \\
Joint explicit & $.8352$ & $.2095$ & $.3414$ & $.1707$ & $.0579$ \\
Unfiltered Gaussian & $.8733$ & $.2503$ & $.3909$ & $.2395$ & $.1192$ \\
\bottomrule
\end{tabular}
\end{table}

Table~\ref{tab:cifar101-additional} reports the full calculation. The filtered
and unfiltered selected-label accuracies are $0.6319$ and $0.6502$. The mean
retained fraction is $0.351061$. All methods use the same selection and
estimation proposal batches within each image.

\subsection{Fixed-cohort CIFAR-10.1 attack}
\label{app:cifar101-attack}

The protocol orders all $2{,}021$ indices by the SHA-256 hash of a fixed salt
and the index, then takes the first $128$. This selection does not use a
model output or image label. All cohort members remain in the primary
denominator. The released radius is positive for $107$ images and nonpositive
for $21$.

Every positive-radius image receives four projected-gradient restarts with
$150$ steps, $256$ Gaussian proposals per step, and temperatures $1$, $0.5$,
and $0.25$. Candidate centers are projected to $[0,1]^d$, while the Gaussian
proposals supplied to the model remain unclipped. The best and final centers
from each restart give eight candidates per positive-radius image. Each
candidate receives a disjoint $50{,}000$-proposal screen. The center with the
largest target margin is retained for that image. The search takes $8852$
seconds and produces twelve positive screens.
The supplement contains every search attempt and the preconfirmation summary.

Each positive screen then receives $250{,}000$ new proposals at the nominal and
shifted endpoints. One Bonferroni family reserves nine pairwise winner tests at
each endpoint and one nominal conditional-probability lower bound for every
possible cohort member. Its size is
$128[2(10-1)+1]=2{,}432$, its familywise error is $0.001$, and its
per-inference error is $4.112\times10^{-7}$. The confirmation takes $466$
seconds. Its result and row file are included in the supplement.

All twelve screened candidates have verified opposite endpoint labels and lie
below fresh lower bounds on the corresponding population substituted radii.
The largest endpoint test value among them is $5.27\times10^{-11}$. Table
\ref{tab:cifar101-fixed-cohort} gives the displacement, fresh radius bound, and
retained-share effect sizes. The primary fixed-cohort attack yield is $12/128$.
Nine of the twelve verified cases occur among the $80$ cohort images that are
correctly classified and have a positive released radius. These are results for
the declared attack and finite cohort. The generic output also contains a
binomial interval for the count. We do not use it because deterministic
hash-ranking does not define an independent probability sample from a model
population. Candidate tensors derived from CIFAR-10.1 images are not
redistributed.

\begin{table}[t]
\centering
\caption{Fresh confirmation for every positive screen in the fixed cohort.
The last two columns give the declared winner's advantage over the other listed
label as a percentage of all retained proposals. Every row is verified.}
\label{tab:cifar101-fixed-cohort}
\begin{tabular}{@{}rrcrrrrr@{}}
\toprule
Cohort & ID & labels & $\|\Delta\|_2$ & $r_L$ & $\|\Delta\|_2/r_L$ & nominal & shifted \\
\midrule
3 & 174 & $6\to2$ & $.1641$ & $.1852$ & $.886$ & $57.47$ & $18.19$ \\
11 & 1586 & $3\to2$ & $.3205$ & $.3707$ & $.864$ & $88.38$ & $29.30$ \\
15 & 879 & $2\to6$ & $.2744$ & $.2912$ & $.942$ & $77.59$ & $29.93$ \\
16 & 1060 & $7\to4$ & $.4674$ & $.4739$ & $.986$ & $94.80$ & $23.82$ \\
19 & 759 & $1\to9$ & $.2920$ & $.3014$ & $.969$ & $78.27$ & $5.80$ \\
37 & 215 & $0\to9$ & $.1325$ & $.1496$ & $.885$ & $47.36$ & $3.88$ \\
51 & 1986 & $8\to0$ & $.1649$ & $.1896$ & $.870$ & $57.37$ & $11.17$ \\
61 & 366 & $9\to2$ & $.3198$ & $.3421$ & $.935$ & $84.24$ & $21.36$ \\
79 & 775 & $1\to9$ & $.1799$ & $.1940$ & $.927$ & $57.96$ & $2.99$ \\
84 & 289 & $9\to8$ & $.3499$ & $.3687$ & $.949$ & $87.21$ & $19.66$ \\
99 & 1267 & $8\to0$ & $.2453$ & $.2544$ & $.964$ & $71.19$ & $15.82$ \\
125 & 791 & $8\to1$ & $.3082$ & $.3530$ & $.873$ & $86.34$ & $7.77$ \\
\bottomrule
\end{tabular}
\end{table}

\subsection{Output-selected boundary confirmation}
\label{app:auditvotes-boundaries}

The boundary protocol specifies twelve candidate images, the search procedure,
and disjoint search, screen, and confirmation seeds. Candidate images came
from the opened full-test output
and were correctly classified examples whose filtered and unfiltered selected
labels differed. Four projected-gradient restarts used $256$ proposals per
step. A fresh $50{,}000$-proposal screen selected at most six positive
alternative-minus-original differences. The selected order, labels, and image
tensors were then fixed before confirmation. Candidate centers were projected
to $[0,1]^d$, and Gaussian proposals remained unclipped.

Each endpoint uses one million fresh proposals under the released Gaussian law
and confidence filter. For a declared winner $w$ and competitor $c$, condition
on the number of retained proposals carrying either label. Under the null
$s_w\leq s_c$, the winner count is binomial with success probability at most
$1/2$. Its upper tail therefore gives an exact one-sided test. One Bonferroni
family contains all $6\mathbin{\times}2\mathbin{\times}9=108$ endpoint tests
and six nominal one-sided Clopper--Pearson lower bounds. The per-inference
error is $0.001/114$. If $L_A$ is a nominal conditional-probability lower
bound, $r_L=\sigma\Phi^{-1}(L_A)$ lower bounds the binary population radius
used by the released formula. Verification requires opposite endpoint winners
and $\|\Delta\|_2<r_L$. Table~\ref{tab:auditvotes-boundaries} reports every
endpoint test.

\begin{table}[H]
\centering
\caption{Independent endpoint confirmation. Count pairs give original-label
and alternative-label retained counts among one million proposals. A verified
row has both declared endpoint winners and a shift below $r_L$.}
\label{tab:auditvotes-boundaries}
\begin{tabular}{@{}rcrrrrc@{}}
\toprule
ID & labels & $\|\Delta\|_2$ & $r_L$ & nominal counts & shifted counts & verified \\
\midrule
4924 & $7\to5$ & $.4577$ & $.5097$ & $52295/965$ & $7182/17611$ & yes \\
5089 & $6\to3$ & $.3116$ & $.3585$ & $32339/2345$ & $9196/17483$ & yes \\
8654 & $2\to3$ & $.3316$ & $.4141$ & $8736/358$ & $2052/4782$ & yes \\
6665 & $4\to3$ & $.2843$ & $.3685$ & $18244/1172$ & $5741/7099$ & yes \\
634 & $7\to4$ & $.3068$ & $.4013$ & $36793/1873$ & $11556/12135$ & no \\
5211 & $7\to3$ & $.2570$ & $.3775$ & $27299/1245$ & $7754/7157$ & no \\
\bottomrule
\end{tabular}
\end{table}

All nominal endpoint maximum $p$-values are below $10^{-300}$. The largest
shifted-endpoint $p$-value among verified rows is
$2.09\times10^{-33}$. Pair 634 has a shifted value
$8.65\times10^{-5}$, above the simultaneous cutoff, and pair 5211 does not
reverse its empirical label ordering. Numerical underflowed values are not
reported as exact zeros. Four population substitution violations are verified
with familywise error at most $0.001$. Each verified shift is also below the
finite-sample radius stored by the released calculation during candidate
selection. The stronger confirmation statement uses fresh nominal samples to
lower bound the corresponding population substituted radius. The raw
confirmation records, analysis, and execution log are included in the
supplement. Since the candidates were selected from opened output, this result
establishes existence but does not estimate prevalence.

\subsection{Confidence-threshold selection}
\label{app:threshold-selection}

The threshold protocol was specified before any multithreshold result was
generated. The declared grid is
$0$, $0.5$, $0.7$, $0.8$, $0.85$, $0.9$, $0.925$, $0.95$, and $0.975$. Each
network evaluation is shared across every threshold. The proposal counts,
error probability, model, checkpoint, and noise scale match the full-test study,
while the base random seed is independent. The first $1{,}000$ ordered test
indices form the development set and the remaining $9{,}000$ form the
confirmation set. The supplement includes the per-image record and summary.

The objective averages explicit-joint certified accuracy over radii $0.25$,
$0.5$, and $0.75$. The development set selects threshold zero. Table
\ref{tab:threshold-selection} reports every candidate and confirms the same
ordering on the untouched indices. This study addresses the released
checkpoint only and does not rule out a benefit from confidence filtering after
different training.

\begin{table}[H]
\centering
\caption{Threshold selection. Dev and confirm are mean explicit-joint
certified accuracies over the three declared radii. Retained is the mean
confirmation proposal fraction.}
\label{tab:threshold-selection}
\begin{tabular}{@{}lrrrrrrrrr@{}}
\toprule
Threshold & $0$ & $.5$ & $.7$ & $.8$ & $.85$ & $.9$ & $.925$ & $.95$ & $.975$ \\
\midrule
Dev & $.400$ & $.395$ & $.376$ & $.355$ & $.344$ & $.324$ & $.308$ & $.283$ & $.240$ \\
Confirm & $.399$ & $.394$ & $.376$ & $.358$ & $.343$ & $.320$ & $.304$ & $.279$ & $.237$ \\
Retained & $1.000$ & $.876$ & $.698$ & $.607$ & $.555$ & $.493$ & $.455$ & $.407$ & $.338$ \\
\bottomrule
\end{tabular}
\end{table}

\subsection{Population feasible-action reanalysis}
\label{app:reach-avoid}

The reach--avoid reanalysis uses the already opened protocol-v2 population.
It contains seven deterministic reference states from each of $64$ layouts.
Analytic products of one-dimensional Gaussian CDF differences give occupancy
and label-event masses for each rectangular union. A bracketed scalar root
locates the first vertical label boundary with Brent tolerances
$2\times10^{-13}$. These are population calculations with standard
floating-point Gaussian CDF evaluations and contain no Monte Carlo confidence
interval. The smallest substituted-radius excess over the located boundary is
$4.19\times10^{-4}$. The smallest boundary excess over the joint radius is
$4.72\times10^{-4}$. The crossing statements use a $2\times10^{-10}$ numerical
margin. They are not outward-interval certificates.

The substituted conditional radius crosses the located boundary at all $448$
states. The joint-mass radius is positive at all states and crosses none. Its
median, fifth percentile, and ninety-fifth percentile are $0.2267$, $0.0197$,
and $0.4719$. The corresponding substituted values are $0.3307$, $0.0320$, and
$0.6540$. The joint radius divided by the boundary distance has median $0.7708$
and maximum $0.8628$. The joint radius divided by the substituted radius has
median $0.6872$.

\subsection{Geometric breadth}
\label{app:random-boxes}

A separate synthetic study crosses dimensions $2,5,10,20$, component
counts $4,8,16$, two geometry families, two Gaussian scales, two normalized
separations, and paired volume settings. The $6{,}144$ cells reuse $384$
independently seeded base geometries and are not independent observations.
Figure~\ref{fig:random-box-main} reports the geometry-level comparison.

\begin{figure}[H]
  \centering
  \includegraphics[width=\linewidth]{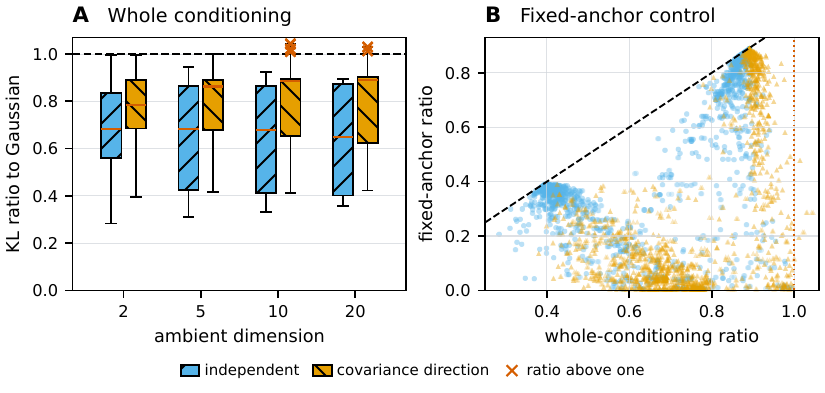}
  \caption{Product-box study. Panel A finds ten covariance-direction KL
  ratios above one from five base geometries. The maximum is $1.0448$, and no
  independent-direction ratio exceeds one. Every fixed-anchor control in Panel
  B stays below one. The study does not estimate prevalence.}
  \label{fig:random-box-main}
\end{figure}

\subsection{Held-out covariance replication}
\label{app:covariance-holdout}

The replication protocol was specified after the earlier study identified the
coherent $16$-box, balanced-volume family and before drawing the new geometry
and direction seeds. It uses dimensions $10$ and $20$, scales $0.25$ and
$0.5$, normalized distances $0.05$ and $0.15$, and $64$ base geometries per
dimension. Each geometry is reused for the fixed scale, distance, and direction
comparisons. The $128$ base geometries, not the $1{,}024$ factor rows, are
independent analysis units.

Fourteen geometries exceed the Gaussian KL rate in the anchor top-covariance
direction and none does so in the independently seeded direction. The
exact one-sided sign test on the paired exceedance indicators has value
$6.10\times10^{-5}$. The covariance-direction maximum exceeds the independent
maximum in every geometry, with median difference $0.06629$. The minimum
fixed-cell Spearman correlation between anchor directional covariance and
finite KL ratio is $0.9941$. Direct KL and the covariance path integral agree
within $1.77\times10^{-15}$. Figure~\ref{fig:covariance-holdout-full} gives the
full local diagnostic. The paired maxima are computed from the complete row
file, which is included with the result summary in the supplement.

\begin{figure}[H]
  \centering
  \includegraphics[width=\linewidth]{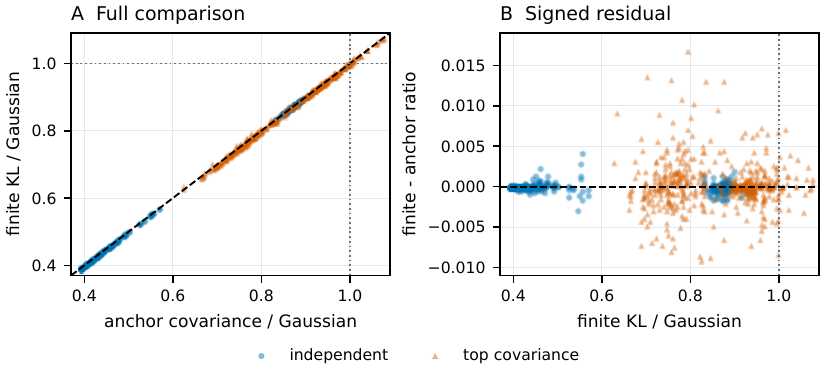}
  \caption{Held-out covariance replication. Panel A compares anchor
  directional covariance and finite KL over all $1{,}024$ factor rows. The
  diagonal is equality and dotted lines mark the Gaussian rate of one. Panel B
  shows the signed finite-minus-anchor ratio, exposing the finite-scale
  discrepancy hidden by the near-diagonal scatter. Twenty-seven repeated
  top-covariance rows from $14$ base geometries exceed one. No independent
  row does. The $128$ base geometries, not factor rows, are the independent
  analysis units. Direct KL and the path integral agree within
  $1.77\times10^{-15}$.}
  \label{fig:covariance-holdout-full}
\end{figure}

The family was selected from earlier results, so this is a targeted replication
rather than a prevalence study. The top-covariance direction is an
adaptive stress direction. The local covariance ratio and finite KL ratio test
Proposition~\ref{prop:local-diagnostic}. They do not certify a uniform
covariance bound or predict a label change.

\subsection{External projected-band confirmation}
\label{app:cifar102}

This study evaluates the geometry-controlled certificate in
Proposition~\ref{prop:diameter-radii} on a learned image classifier. The filter
was selected using CIFAR-10 training images and fixed before CIFAR-10.2 was
accessed. It retains a raw noisy proposal $z$ when
\[
  0.0625\leq |\langle u,z\rangle+0.1072130481|\leq0.25,
  \qquad \sigma=0.25.
\]
The unit vector $u$ and its exact representation are included in the
supplement.
For $T=\langle u,z\rangle+b$, retention implies $T\in[-\beta,\beta]$.
Popoviciu gives $\operatorname{Var}(T\mid R=1)\leq\beta^2\leq\sigma^2$.
The orthogonal Gaussian coordinates are unchanged and independent of the
retention event. Hence
$\operatorname{Cov}(Q_c)\preceq\sigma^2I$ for every center $c$. This is a
global analytic premise, not an empirical covariance estimate.

The official CIFAR-10.2 test set \citep{lu2020harder} contains $2{,}000$ images
with $200$ from each class. The supplement records its exact revision and file
checksum. Images are evaluated in stored order. The AuditVotes source,
checkpoint, and architecture are those in
Appendix~\ref{app:auditvotes}. The supplement records the dataset checksum and
contains the confirmation protocol and runner.

Each image receives four independent streams. Filtered and unfiltered label
selection use $10{,}000$ raw proposals each. Their independent estimation
streams use $100{,}000$ proposals each. Gaussian proposals are not clipped.
The filter is applied before model inference and rejected proposals are not
replaced. All images remain in the denominator. One Bonferroni family assigns
error $0.001$ across $64{,}000$ one-sided Clopper--Pearson bounds. These cover
conditional label probabilities, joint retained-label masses, the
independently selected joint runner used for an upper radius bound, and the
unfiltered probabilities.

Filter selection and a $200$-image replication used CIFAR-10 training images
only. These development results do not enter the reported CIFAR-10.2 evidence.
The supplement contains the development records, the final protocol, and the
complete execution history.

At radius $0.2$, the original forward-KL, joint-mass, and unfiltered calculations
correctly certify $829/2000$, $725/2000$, and $948/2000$ images. Forward KL
certifies $104$ images that joint mass does not, with none in the reverse
direction. The paired exact one-sided value is
$4.93\times10^{-32}$. This value is secondary evidence rather than sampling
uncertainty for the finite census. Among correctly selected filtered labels,
$329$ have $r_{\mathrm{cov},L}-r_{\mathrm{mass},U}\geq0.01$ under the same
simultaneous confidence event. This compares two certificate formulas. It is
not a lower bound on a difference between true robust radii.

\begin{table}[t]
\centering
\caption{CIFAR-10.2 results by class. Retained is the mean estimation-stream
fraction. F-acc and U-acc are filtered and unfiltered selected-label accuracy.
KL and Joint are correct certified fractions at radius $0.2$ in the original
forward-KL analysis. Strong counts correct images whose forward-KL lower
radius exceeds the joint-mass upper radius by at least $0.01$.}
\label{tab:cifar102-classwise}
\begin{tabular}{@{}lrrrrrr@{}}
\toprule
Class & Retained & F-acc & U-acc & KL & Joint & Strong \\
\midrule
airplane   & $.420$ & $.505$ & $.505$ & $.280$ & $.250$ & $26$ \\
automobile & $.418$ & $.640$ & $.640$ & $.420$ & $.385$ & $35$ \\
bird       & $.437$ & $.505$ & $.510$ & $.275$ & $.210$ & $34$ \\
cat        & $.424$ & $.545$ & $.540$ & $.250$ & $.190$ & $44$ \\
deer       & $.432$ & $.585$ & $.585$ & $.340$ & $.275$ & $34$ \\
dog        & $.415$ & $.660$ & $.655$ & $.490$ & $.420$ & $34$ \\
frog       & $.437$ & $.665$ & $.665$ & $.470$ & $.410$ & $35$ \\
horse      & $.425$ & $.745$ & $.740$ & $.520$ & $.490$ & $32$ \\
ship       & $.437$ & $.770$ & $.770$ & $.585$ & $.535$ & $29$ \\
truck      & $.430$ & $.695$ & $.690$ & $.515$ & $.460$ & $26$ \\
\midrule
all        & $.428$ & $.632$ & $.630$ & $.415$ & $.363$ & $329$ \\
\bottomrule
\end{tabular}
\end{table}

The mean retained fraction is $0.427570$. Filtered and unfiltered selected-label
accuracies are $0.6315$ and $0.6300$. The filtered streams make $94{,}064{,}290$
model calls from $220$ million proposals. The unfiltered streams make $220$
million calls. The run took $26{,}992$ seconds on one RTX 4070 Ti with PyTorch
2.2.2. The supplement contains the manifest, complete row file, and summary.
Ordinary unfiltered smoothing remains stronger at radius $0.2$. This finite
census covers one checkpoint, noise scale, and training-selected filter, not
other models, filters, or systems.

\subsection{Conditional divergence reanalysis}
\label{app:renyi-reanalysis}

We applied Proposition~\ref{prop:diameter-radii} to the same stored counts
after the original experiment. The filter, predictions, and confidence bounds
are unchanged. The finite order set is
\[
 \{1\}\cup\{1+k/20\mid k=1,\ldots,80\}\cup\{6,8,12,16,32,64\}.
\]
For each image, we maximize the valid radius over these orders using its
original conditional lower and upper probability bounds. This choice adds no
confidence events or model evaluations. The covariance premise is global,
so $R=\infty$ and $\Lambda=1$. This is a deterministic reanalysis, not a new
independent confirmation.

\begin{table}[t]
\centering
\caption{Correctly certified CIFAR-10.2 images out of $2{,}000$. All filtered
methods use identical counts and selected labels. The unfiltered predictor
uses its separate original streams. Matched calls uses a random subset of
those streams with the same model-call count as the filtered methods.}
\label{tab:renyi-reanalysis}
\begin{tabular}{@{}lrrr@{}}
\toprule
Certificate & $r\geq0.2$ & $r\geq0.3$ & $r\geq0.5$ \\
\midrule
Conditional forward KL & $829$ & $0$ & $0$ \\
Conditional reverse KL & $888$ & $698$ & $329$ \\
Conditional R\'enyi & $892$ & $713$ & $381$ \\
Joint mass & $725$ & $417$ & $0$ \\
Unfiltered, original & $948$ & $814$ & $529$ \\
Unfiltered, matched calls & $945$ & $806$ & $522$ \\
\bottomrule
\end{tabular}
\end{table}

Table~\ref{tab:renyi-reanalysis} separates the effect of reversing KL from
the additional improvement obtained by optimizing the R\'enyi order. The
forward-KL radius cannot exceed $0.25\sqrt{2\log2}=0.29435\ldots$ for this
covariance bound. Reverse KL removes that ceiling. At radius $0.2$, R\'enyi
certifies $167$ correct images that joint mass does not, with none in the
reverse direction. For $877$ correctly selected labels, the R\'enyi lower
radius exceeds the joint-mass upper radius by at least $0.01$. These bounds
share the original familywise error $0.001$ and compare certificate formulas,
not true robust radii.

\paragraph{Outward numerical validation.}
We recomputed the full-count conditional R\'enyi, reverse-KL, joint-mass, and
unfiltered certificates using Arb interval arithmetic at $128$-bit precision.
Each binomial endpoint is accepted only after its tail inequality is proved
against the exact rational error $1/64{,}000{,}000$. A positive-term binomial
recurrence bounds its remaining terms by a geometric series. Gaussian
quantiles and the finite-order divergence formulas are rounded outward.
The certified counts at $0.2$, $0.3$, and $0.5$, and all $877$ separation
statements, are unchanged. The largest change to a conditional probability
endpoint is below $3.51\times10^{-12}$. Reporting thresholds are compared as
exact rational numbers. This validates inference from the stored integer
counts. It does not validate noise generation, filter evaluation, neural
inference, the matched-call calculation, or the separate AuditVotes endpoint
study. 

\begin{figure}[t]
  \centering
  \includegraphics[width=\linewidth]{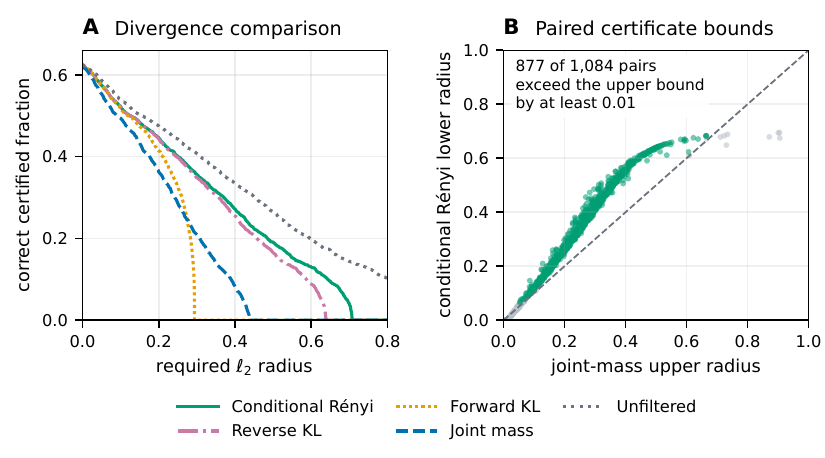}
  \caption{Conditional divergence certificates on the retained CIFAR-10.2
  counts. Panel A includes all $2{,}000$ images and displays the forward-KL
  ceiling and its removal by reverse KL. Panel B shows all $1{,}084$ correctly
  selected labels with finite paired bounds. Green points have a conditional
  R\'enyi lower radius at least $0.01$ above their joint-mass upper radius.
  The dashed diagonal denotes equality.}
  \label{fig:cifar102-strong-separation}
\end{figure}

\paragraph{Matching model calls.}
The filtered streams evaluate only retained proposals. To compare equal model
call counts, we draw multivariate-hypergeometric subsets of each original
unfiltered count vector. Each subset size equals the corresponding filtered
stream's model-call count, giving $94{,}064{,}290$ calls in total. Selection
and estimation subsets are drawn separately with seed $20260924$. Their sizes
depend on independent filtered streams, not on unfiltered labels. Thus they
represent uniformly sampled subsets of the unfiltered IID proposals. We
reselect labels on the smaller selection batch and use the original per-tail
error $1.5625\times10^{-8}$ for $20{,}000$ one-sided binomial bounds. Six selected labels
change. Unfiltered smoothing still certifies more images at all three tabulated
radii. This is one randomized count-level comparison, not a new network or
timing run. Its confidence family is separate from the original evaluation.
Its familywise error is at most $0.0003125$, and the combined error is at most
$0.0013125$.

\end{document}